\documentclass[letterpaper,twocolumn,10pt]{article}
\usepackage{usenix-2020-09}

\usepackage{times}
\usepackage{amsmath,amsthm,amsfonts,amssymb}
\usepackage{xspace}
\usepackage{graphicx}

\theoremstyle{plain}
\newtheorem{theorem}{Theorem}

\definecolor{commentgreen}{rgb}{0, 0.5, 0}
\usepackage[
  ruled,
  vlined,
  linesnumbered,
]{algorithm2e}
\DontPrintSemicolon{}
\SetKwComment{Comment}{$\triangleright$\ }{}

\SetCommentSty{mycommfont}

\let\oldnl\nl
\newcommand{\nonl}{\renewcommand{\nl}{\let\nl\oldnl}}

\usepackage{float}
\usepackage{wrapfig}
\usepackage[english]{babel}
\usepackage{subfig}
\usepackage{xcolor}
\usepackage{colortbl}
\usepackage{pifont}
\usepackage[inline]{enumitem}
\usepackage{multirow}
\usepackage{fancybox} %
\usepackage[small, compact]{titlesec}
\usepackage[textfont={it}, font={small,bf}]{caption}
\usepackage[normalem]{ulem}
\usepackage{authblk}
\usepackage{comment}
\usepackage{booktabs}
\usepackage{tabularx}
\usepackage{balance}

\usepackage{tikz}
\usetikzlibrary{positioning,fit,calc}

\newcommand*\blackcircled[1]{\tikz[baseline=(char.base)]{
		\node[shape=circle,draw,fill=black,inner sep=0pt] (char) {\textcolor{white}{#1}};}}

\theoremstyle{definition}

\newtheorem*{defi*}{Definition}

\usepackage{hyperref}
\hypersetup{
	colorlinks=true,      %
	linkcolor=blue,       %
	citecolor=magenta,    %
	filecolor=cyan,       %
	urlcolor=red          %
}

\usepackage{titling}
\date{}

\usepackage{listings}
\usepackage{courier}
\newenvironment{denseitemize}{
	\begin{itemize}[topsep=2pt, partopsep=0pt, leftmargin=1.5em]
		\setlength{\itemsep}{2pt}
		\setlength{\parskip}{0pt}
		\setlength{\parsep}{0pt}
	}{\end{itemize}}

\usepackage{peanutgallery}

\begin{document}

	\date{}

  \title{\Large \bf Kareus: Joint Reduction of Dynamic and Static Energy in Large Model Training}

	\author{
		\rm Ruofan Wu \qquad Jae-Won Chung \qquad Mosharaf Chowdhury
		\\
		\itshape{University of Michigan}
		} %

  \pagestyle{plain}

  \pagenumbering{gobble} %

  \maketitle

  \begin{abstract}
  The computing demand of AI is growing at an unprecedented rate, but energy supply is not keeping pace.
As a result, energy has become an expensive and contended resource that requires explicit management and optimization.
Although recent works have made significant progress in large model training optimization, they focus on optimizing either dynamic or static energy consumption.

We find that fine-grained kernel scheduling and frequency scaling \emph{jointly} and \emph{interdependently} impact both dynamic and static energy consumption.
Based on this finding, we design Kareus, a training system that pushes the time--energy tradeoff frontier by optimizing both aspects.
Kareus decomposes the intractable joint optimization problem into local, partition-based subproblems.
It then uses a multi-pass multi-objective optimization algorithm to find execution schedules that push the time--energy tradeoff frontier.
Compared to the state of the art, Kareus reduces training energy by up to 28.3\% at the same training time, or reduces training time by up to 27.5\% at the same energy consumption.\footnote{Kareus is open-source at \url{https://github.com/ml-energy/kareus}.}

  \end{abstract}

  \section{Introduction}

Today, energy is the ultimate bottleneck for scaling AI~\cite{bloombergnef25,cbre2025,inferencemax-blog25}.
The energy demand of training large models and serving them to millions is growing at an unprecedented rate~\cite{patterson2021carbon,llama3-arxiv24,hamilton2024constraint}, with projections indicating that by 2035, nearly 10\% of US electricity demand could be from datacenters~\cite{bloombergnef25}.
However, procuring energy at scale is slow, e.g., three years for natural gas and five to ten years for nuclear~\cite{eia-utility-data}.
This mismatch makes energy an expensive and contended resource that must be explicitly budgeted, managed, and allocated with efficient systems and optimization methods.

In this context, large model training is a key target for optimization; a single training run can consume enough energy to power more than 24,000 average US households for a month~\cite{llama3-arxiv24,us-household}.
To understand the current state of large model training optimizations, we analyze them through the lens of \emph{dynamic} energy (consumed by actual work) and \emph{static} energy (consumed at all times regardless of work) consumption of GPUs (\S\ref{sec:background}).
Practically, dynamic energy can be reduced by lowering GPU frequency or activity, while static energy can be reduced by shortening iteration time.
Perseus~\cite{perseus-sosp24} reduces dynamic energy by scaling the frequency of computations off the critical path, but does not reschedule kernels.
Recent works in fine-grained kernel scheduling~\cite{nanoflow-osdi25,tokenweave-mlsys26,vllm-dbo-docs,domino-arxiv24,deepseek-v3-arxiv24} reduce static energy by overlapping computation and communication, but they ignore dynamic energy and frequency scaling.
Therefore, we ask: can we combine both approaches to \emph{jointly} reduce dynamic and static energy?

To answer this question, we study how \emph{execution schedules} affect the time and energy consumption of large model training (\S\ref{sec:analysis}).
Any large model training iteration boils down to computation and communication kernels launched on the GPU.\@
We define an \emph{execution schedule} as the combination of three factors:
(1) the \emph{timing} of when communication kernels are launched within a sequence of computation kernels,
(2) the \emph{number of GPU Streaming Multiprocessors (SMs)} allocated to communication kernels, and
(3) the \emph{GPU frequency}.

We demonstrate that these factors \emph{jointly} and \emph{interdependently} determine time and energy consumption.
Different execution schedules lead to different time and energy consumption---varying by as much as 3.29$\times$---even when the total amount of work remains the same.
The optimal schedule is achieved by carefully balancing the resource consumption of computation and communication kernels, which depends on the interplay of all three factors.
Crucially, \emph{the best SM allocation and kernel launch timing depend on GPU frequency}; even for the same sequence of kernels, we cannot use the same SM allocation and launch timing at different frequencies.
Intuitively, frequency changes how overlapping kernels contend for resources: lowering frequency slows computation while leaving memory and communication bandwidth largely unchanged, thereby reshaping the optimal SM allocation and kernel launch timing.
Existing solutions for kernel scheduling~\cite{nanoflow-osdi25,tokenweave-mlsys26,vllm-dbo-docs,domino-arxiv24,deepseek-v3-arxiv24} and frequency scaling~\cite{perseus-sosp24} optimize disjoint subsets of these factors, and combining them naively is suboptimal for pushing the time--energy tradeoff frontier.

Based on these observations, we present \emph{Kareus}, an energy-efficient large model training system that automatically finds the best execution schedule by jointly optimizing SM allocation, launch timing, and GPU frequency (\S\ref{sec:design}).
For all kernels in the training iteration, it identifies (1) the right number of SMs for communication kernels, and (2) the right amount of computation--communication overlap, both in a frequency-specific manner.
A collateral benefit is that executions on the critical path also speeds up due to improved overlap, reducing training time and thus shifting the entire time--energy tradeoff frontier toward the origin.

Unfortunately, jointly optimizing all three factors via exhaustive search is impractical.
The search space is prohibitively large to profile each configuration.
Worse, profiling heats up the GPU, and without sufficient cooling intervals, one configuration's measurement affects subsequent ones.

Kareus addresses these challenges by introducing the \emph{partitioned overlap} execution model.
The execution graphs of forward and backward passes are divided into fixed partitions, decomposing the global optimization problem into local subproblems.
Partitioned overlap generalizes recent \emph{nanobatching} techniques~\cite{nanoflow-osdi25,tokenweave-mlsys26,vllm-dbo-docs,domino-arxiv24,deepseek-v3-arxiv24} and adds fine-grained control over all three execution schedule factors, enabling precise control over the overlap between each communication kernel and its surrounding computation.
For each partition, Kareus uses a multi-pass multi-objective optimization algorithm to identify candidates on the time--energy tradeoff frontier.
It then hierarchically composes local (partition-level) frontiers into a global (iteration-level) one for the entire training iteration.

We implement Kareus by integrating with Perseus~\cite{perseus-sosp24} and Megatron-LM~\cite{megatronlm-sc21} (\S\ref{sec:implementation}).
Kareus's optimizer is informed by time and energy profiling results from our \emph{thermally stable} profiler, which minimizes thermal interference between different configurations during profiling.
Once the best execution schedule is identified, Kareus interacts with Perseus and Megatron-LM to realize the plan throughout the entire training execution.

We evaluate Kareus across 14 representative workloads, including real testbed training on Llama 3.2 3B and Qwen 3 1.7B, and large-scale emulation on Llama 3.3 70B (\S\ref{sec:evaluation}).
Compared to the state-of-the-art Perseus, Kareus reduces energy by up to 28.3\% under the same time budget, or reduces time by up to 27.5\% under the same energy budget.

In summary, we make the following contributions:
\begin{denseitemize}
  \item We show that SM allocation, launch timing, and GPU frequency jointly determine time and energy consumption, and that existing solutions optimizing subsets of these factors and naive combinations thereof are suboptimal.

  \item We design Kareus around the \emph{partitioned overlap} execution model, decomposing the global optimization problem into tractable local subproblems.

  \item We evaluate Kareus on large model training workloads, demonstrating significant time and energy reduction compared to the state-of-the-art.
\end{denseitemize}

  \section{Background}\label{sec:background}

We begin by providing a brief background on GPU power consumption (\S\ref{sec:background-power}).
Then, we compare existing training systems in terms of their goals and techniques (\S\ref{sec:background-comparison}), and discuss implications for training iteration time and energy (\S\ref{sec:background-breakdown}).

\subsection{GPU Power Consumption}\label{sec:background-power}

A GPU's power consumption can be divided into two components: \emph{dynamic power and static power}~\cite{gpuwattch,flipflop-icse26,perseus-sosp24}.
Dynamic power is the power consumed by the chip's \emph{compute and memory activity}; it is proportional to core frequency and the square of voltage.
In contrast, static power is consumed by \emph{all parts of the chip} at all times regardless of utilization or activity.

\vspace{0.2em}
\subsection{Existing Training Systems}\label{sec:background-comparison}

\begin{figure}[t]
  \begin{center}
    \subfloat[1F1B pipeline with no frequency scaling]{
      \includegraphics[width=0.85\columnwidth]{
        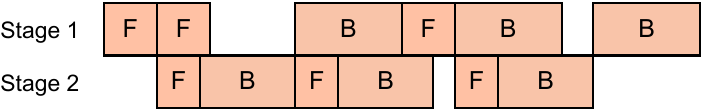
      }\label{fig:background-system-megatron}
    }

    \subfloat[1F1B pipeline with frequency scaling by Perseus]{
      \includegraphics[width=0.85\columnwidth]{
        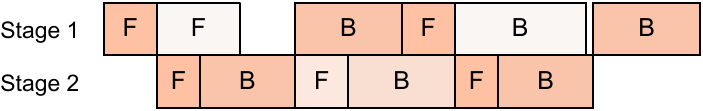
      }\label{fig:background-system-perseus}
    }
  \end{center}
  \vspace{-1.0em}
  \caption{
    Existing training systems running the 1F1B pipeline schedule~\cite{megatronlm-sc21}.
    Redder colors indicate higher GPU frequency and power draw.
    (a) Megatron-LM~\cite{megatronlm-sc21} does not perform any frequency scaling, whereas (b) Perseus~\cite{perseus-sosp24} scales GPU frequencies of non-critical forward and backward computations to reduce energy consumption while maintaining the same latency.
  }
\end{figure}

\begin{figure}[t]
  \begin{center}
    \subfloat[Sequential kernel execution]{
      \includegraphics[width=0.90\columnwidth,clip,trim=0 65 0 0]{
        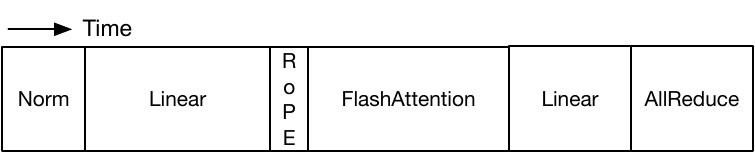
      }\label{fig:background-execution-sequential}
    }

    \subfloat[Nanobatching]{
      \includegraphics[width=0.90\columnwidth,clip,trim=0 0 0 80]{
        figures/background/execution_combined.pdf
      }\label{fig:background-execution-nanobatching}
    }
  \end{center}
  \vspace{-1.0em}
  \caption{
    Transformer~\cite{transformer-neurips17} Attention layer with tensor parallelism run with different execution models.
    (a) The sequential kernel execution model only runs one kernel at a time strictly following data dependencies.
    (b) Nanobatching~\cite{domino-arxiv24,deepseek-v3-arxiv24,nanoflow-osdi25,tokenweave-mlsys26,vllm-dbo-docs} splits a pipeline microbatch into two nanobatches (illustrated with different colors) and staggers their execution, creating opportunities to overlap communication and computation.
  }
\end{figure}

Large model training relies on multi-GPU parallelism, which distributes model weights and computation while introducing communication.
On each GPU, kernel scheduling determines how computation and communication are executed.
Figure~\ref{fig:background-system-megatron} illustrates Megatron-LM's~\cite{megatronlm-sc21} implementation of the 1F1B pipeline schedule.
Each pipeline microbatch (boxes labeled F and B) contains multiple Transformer~\cite{transformer-neurips17} blocks, each with an Attention layer and a feed-forward (MLP) layer.
Figure~\ref{fig:background-execution-sequential} shows the sequential kernel execution model adopted by Megatron-LM~\cite{megatronlm-sc21}, where kernels execute one after another following data dependencies.

Beyond parallelizing work across devices~\cite{megatronlm-sc21,megascale-nsdi24,alpa-osdi22,llama3-isca25,partir-asplos25,cornstarch-arxiv25}, recent works have optimized latency with \emph{nanobatching}~\cite{domino-arxiv24,deepseek-v3-arxiv24,nanoflow-osdi25,tokenweave-mlsys26,vllm-dbo-docs}, where a single pipeline microbatch is split into two equal-sized \emph{nanobatches} with no data dependencies between them.
This creates opportunities to overlap communication and computation kernels from different nanobatches,\footnote{Computation and communication exert different GPU resources, so overlapping them increases overall GPU utilization and can lead to speedup.} as shown in Figure~\ref{fig:background-execution-nanobatching}.

A separate line of work has optimized the energy consumption of large model training by scaling GPU frequencies~\cite{zeus-nsdi23,envpipe-atc23,perseus-sosp24}.
Notably, Perseus~\cite{perseus-sosp24} dynamically reduces the GPU frequencies of microbatches off the critical path of computation, as shown in Figure~\ref{fig:background-system-perseus}, to reduce energy consumption while keeping overall training iteration time the same.
However, it follows the sequential kernel execution model (Figure~\ref{fig:background-execution-sequential}), missing opportunities from fine-grained kernel scheduling like nanobatching.

\subsection{Breaking Down Training Energy Consumption}\label{sec:background-breakdown}

\begin{table}[t]
  \footnotesize
  \centering
  \caption{
    Training iteration time (seconds) and energy breakdown (Joules) of Megatron-LM, Nanobatching, and each combined with Perseus, training Qwen 3 1.7B on 16 NVIDIA A100 GPUs.
    Megatron-LM achieves 99.0 TFLOP/s/GPU.
  }\label{tab:background-system-comparison}
  \vspace{-0.5em}
  \begin{tabular}{lrrrr}
    \toprule
      & \begin{tabular}{@{}r@{}}Iteration\\time\end{tabular}
      & \begin{tabular}{@{}r@{}}Static\\energy\end{tabular}
      & \begin{tabular}{@{}r@{}}Dynamic\\energy\end{tabular}
      & \begin{tabular}{@{}r@{}}Total\\energy\end{tabular} \\
    \midrule
    Megatron-LM            & 5.60 & 5,372 & 21,374 & 26,745 \\
    Megatron-LM + Perseus  & 5.60 & 5,374 & 19,531 & 24,905 \\
    Nanobatching           & 5.31 & 5,096 & 21,445 & 26,541 \\
    Nanobatching + Perseus & 5.37 & 5,160 & 19,729 & 24,889 \\
    \bottomrule
  \end{tabular}
\end{table}

We can better understand the energy consumption of existing training systems and techniques by breaking down energy into static and dynamic components.
Table~\ref{tab:background-system-comparison} presents such a breakdown along with iteration time for Qwen 3 1.7B~\cite{qwen3-arxiv25} training on 16 NVIDIA A100 GPUs.\footnote{Trained with 8 microbatches, each with size 16 and sequence length 4K, using pipeline parallelism 2, context parallelism 2, and tensor parallelism 4. The same experiment also appears in Table~\ref{tab:e2e-leftmost}, Table~\ref{tab:e2e-iso}, and Figure~\ref{fig:eval-frontier-real}.}
Static energy is static power\footnote{Static power is set to be the GPU's power draw when it is in \emph{ready} state (power state P0) without running any significant computations.} multiplied by iteration time, and dynamic energy is static energy subtracted from total energy.

Compared to baseline Megatron-LM (first row), Nanobatching (third row) reduces iteration time, which directly lowers static energy.
Dynamic energy is slightly higher because nanobatching incurs additional GPU activity from more memory accesses and extra gradient accumulations per nanobatch.
Perseus applied to Megatron-LM (second row) reduces dynamic energy through explicit frequency scaling while keeping iteration time nearly the same, so static energy remains unchanged.
Applying Perseus on top of Nanobatching (fourth row) combines both benefits: reduced static energy from shorter time and dynamic energy from frequency scaling.

However, we find that there is further opportunity to reduce energy beyond what is achieved by simply combining Nanobatching and Perseus.
To understand this gap, we analyze how communication and kernel scheduling affect energy consumption (\S\ref{sec:analysis}).
The opportunities we find motivate the design of Kareus, which jointly optimizes kernel scheduling and frequency scaling (\S\ref{sec:design}).

  \section{Energy Impact of Execution Schedules}\label{sec:analysis}

We investigate the relationship between energy consumption and execution schedule.
We first discuss how different execution schedules for the same \emph{work} can lead to different energy consumption (\S\ref{sec:analysis-background}), and then analyze the impact of key factors through case studies (\S\ref{sec:analysis-case-study}).
This reveals the necessity of joint control and the opportunity for significant energy savings (\S\ref{sec:analysis-joint-control}).

\vspace{0.25em}
\subsection{Execution Schedule and Energy Consumption}\label{sec:analysis-background}

Regardless of \emph{how} work (computation, memory access, and communication) is performed, the total amount of work done is the same.
However, different execution schedules can still lead to different energy consumption.

The distinction between dynamic and static power is key to understanding this.
Dynamic energy reflects the total work done on the hardware (i.e., transistors switching from computation, memory access, and communication).
At \emph{the same GPU frequency}, it remains largely constant across execution schedules.
In contrast, all parts of the hardware consume static power as long as they are powered on.
Compute components (e.g., SMs) account for a significant portion of chip area and thus static power.
When GPU resources are underutilized, static power is still dissipated without as much useful work being done, increasing total energy consumption.

In sum, GPU frequency primarily influences dynamic energy, whereas execution schedules largely determine static energy by changing end-to-end execution time.
However, these factors are interdependent and cannot be decoupled and optimized separately, as we show next.

\vspace{0.25em}
\subsection{Interdependent Factors in Execution Schedules}\label{sec:analysis-case-study}

Execution schedules differ in which kernels run concurrently on the GPU, how they split GPU resources, and at what GPU frequency they run.
We study the impact of execution schedules on time and energy consumption by varying three key factors that correspond to each: (1) communication kernel launch timing, (2) number of SMs allocated to the communication kernel, and (3) GPU frequency.
This definition of execution schedule generalizes nanobatching.
The original nanobatching model (1) launches communication kernels as soon as possible, (2) uses default communication kernels like NCCL optimized for sequential execution (which may use excessive SMs assuming no concurrent kernels), and (3) does not perform GPU frequency scaling.

\begin{figure*}[t!]
  \begin{center}
    \includegraphics[width=0.35\textwidth]{
      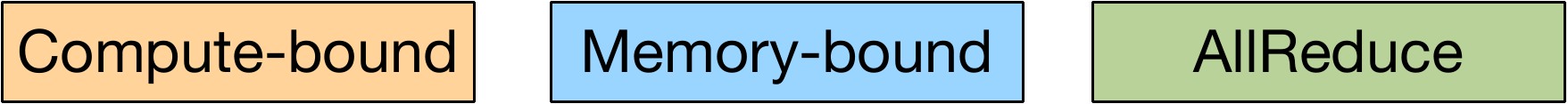
    }
  \end{center}
  \vspace{-2.3em}

  \begin{center}
    \begin{minipage}[t]{0.48\textwidth}
      \centering
      \subfloat[2 SMs lead to exposed communication time]{
        \includegraphics[width=0.95\columnwidth]{
          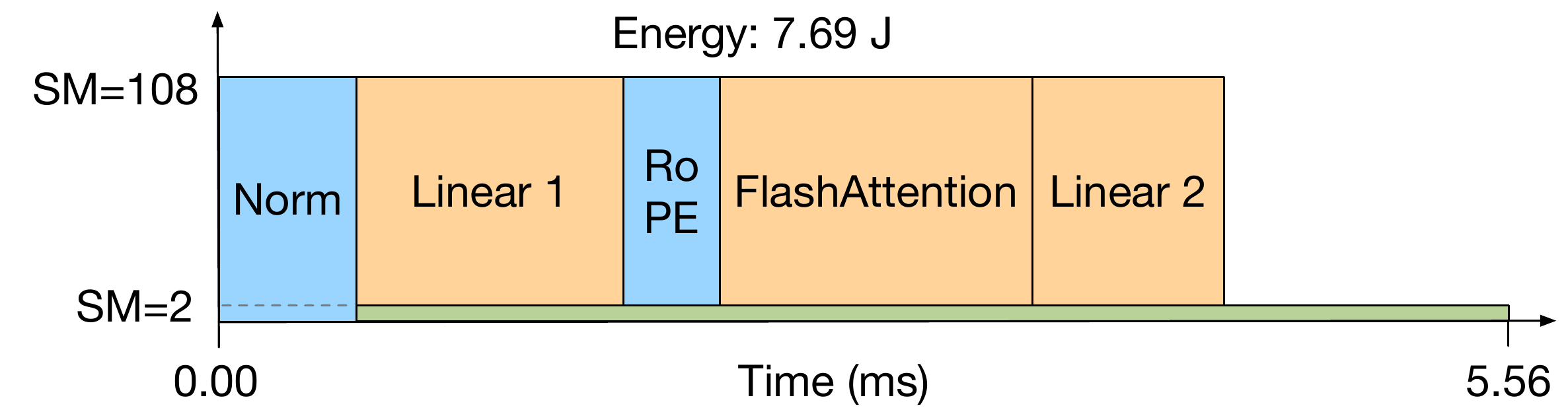
        }\label{fig:analysis-vary-sm-2}
      }
      \vspace{-0.3em}

      \subfloat[4 SMs eliminate exposed communication time]{
        \includegraphics[width=0.95\columnwidth]{
          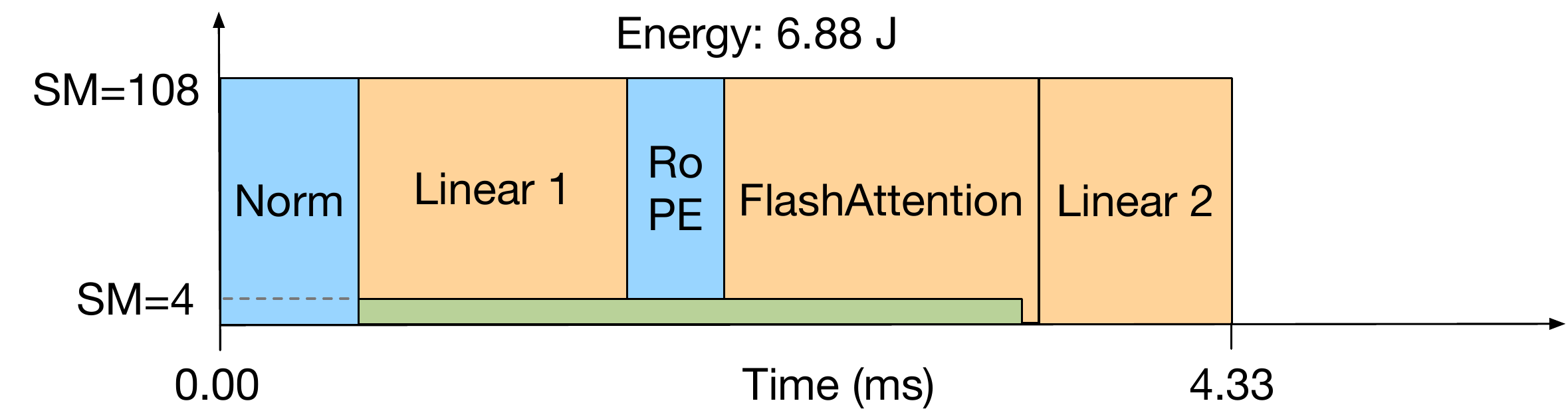
        }\label{fig:analysis-vary-sm-4}
      }
      \vspace{-0.3em}

      \subfloat[20 SMs reduce communication latency but lead to compute interference]{
        \includegraphics[width=0.95\columnwidth]{
          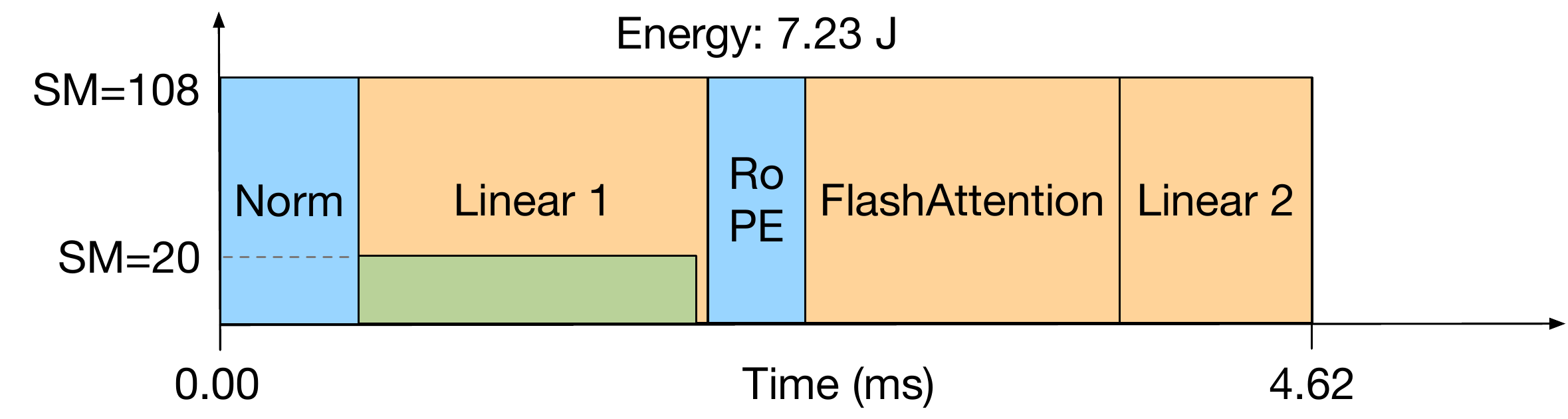
        }\label{fig:analysis-vary-sm-20}
      }
    \end{minipage}
    \hfill
    \begin{minipage}[t]{0.48\textwidth}
      \subfloat[Same as (b) but communication was launched with \texttt{Norm}]{
        \includegraphics[width=0.95\columnwidth]{
          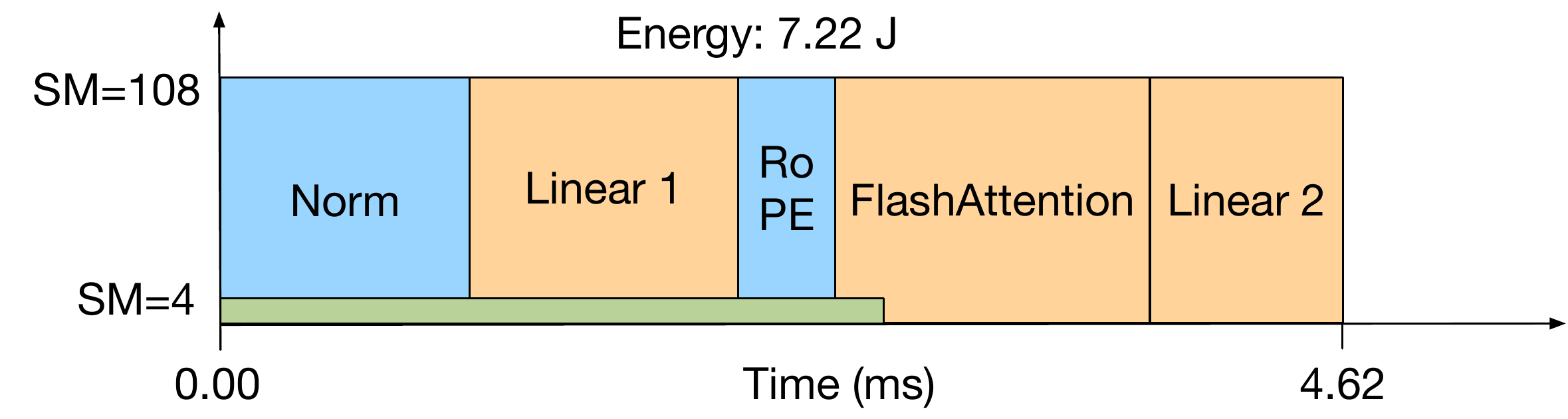
        }\label{fig:analysis-start-norm-1410}
      }
      \vspace{-0.3em}

      \subfloat[Same as (d) but runs at a lower GPU frequency (1,100 MHz)]{
        \includegraphics[width=0.95\columnwidth]{
          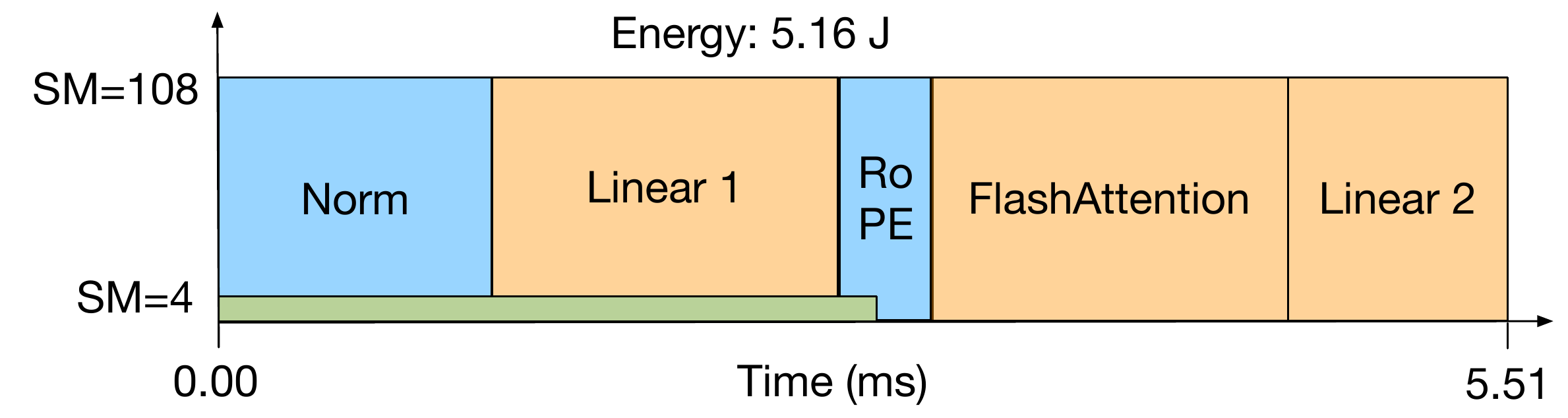
        }\label{fig:analysis-start-norm-1100}
      }
      \vspace{-0.3em}

      \subfloat[Lower frequency (1,100 MHz) changes the energy-optimal schedule]{
        \includegraphics[width=0.95\columnwidth]{
          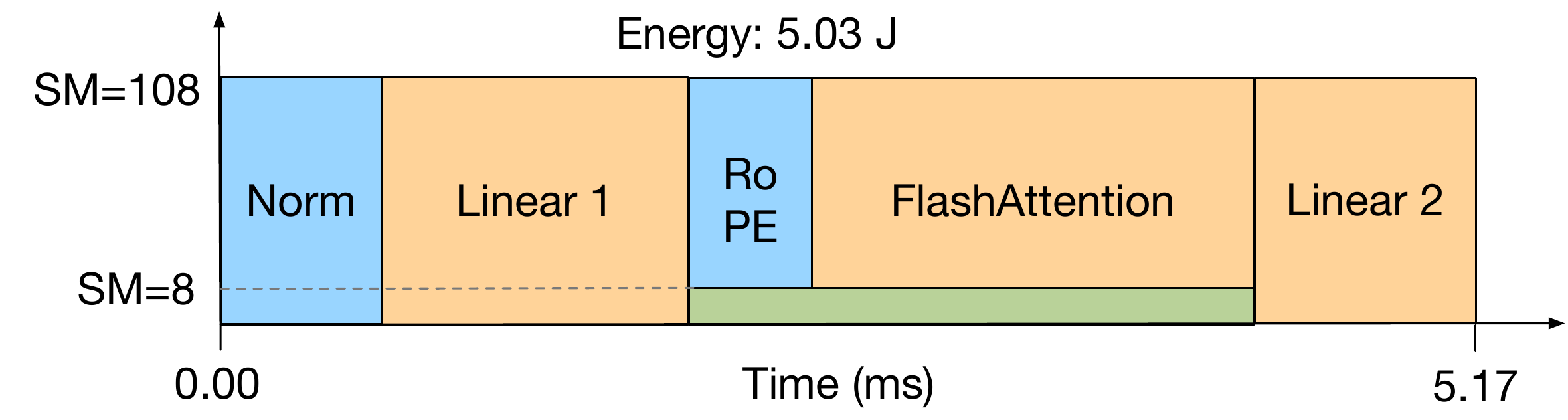
        }\label{fig:analysis-start-rmsnorm-1100}
      }
    \end{minipage}
  \end{center}
  \vspace{-1em}
  \caption{
    The time and total energy consumption of execution schedules for one Transformer Attention layer forward pass with varying SM allocation, communication launch timing, and GPU frequency.
    (a)--(c) show the effect of allocating different numbers of SMs to the communication kernel at 1,410 MHz, with (b) being the energy-optimal schedule.
    (d) is the same as (b), except that communication was launched earlier together with \texttt{Norm}.
    Finally, (e) and (f) run at a lower GPU frequency of 1,100 MHz.
    (e) is the same as (d) other than frequency, whereas (f) is the energy-optimal schedule at 1,100 MHz, which is different from the schedule in (b).
  }\label{fig:analysis-vary-schedules}
  \vspace{-1em}
\end{figure*}

\begin{table*}[th]
  \footnotesize
  \centering
  \caption{
    Summary of factors and observations from case studies on the energy impact of execution schedules.
  }\label{tab:analysis-observations}
  \vspace{-0.5em}
  \begin{tabular}{p{0.12\columnwidth}p{0.38\columnwidth}p{1.38\columnwidth}}
    \toprule
    \textbf{Factor} & \textbf{Observation} & \textbf{Explanation} \\
    \midrule
    SM\newline allocation
      & A middle-ground sweet spot exists
      & Too few SMs for a kernel can lead to periods where only one small kernel is running, wasting static power; too many SMs slow down other kernels without significantly speeding up the target kernel. \\
    \addlinespace
    Launch\newline timing
      & Resource demands of kernels running together matter
      & Communication, which can be memory-bound, running together with memory-bound operations (e.g., \texttt{Norm}, \texttt{RoPE}) causes bandwidth contention and slowdown for everyone. \\
    \addlinespace
    GPU\newline frequency
      & Changes relative kernel resource demands
      & Lower frequency makes kernels relatively more compute-bound, changing which kernels benefit from overlapping; it also amplifies static power's relative impact, making exposed communication more harmful. \\
    \addlinespace
    All three
      & Time \& energy impacts of all factors are interdependent
      & The energy-optimal SM allocation and launch timing depend on the GPU frequency; changing one factor shifts the energy-optimal configurations of the others. \\
    \bottomrule
  \end{tabular}
  \vspace{-1em}
\end{table*}

Figure~\ref{fig:analysis-vary-schedules} shows the timelines and total energy consumption of different execution schedules for a single Transformer Attention layer of Llama 3.2 3B~\cite{llama3-arxiv24}, with tensor parallelism degree 4 on four fully-connected NVIDIA A100 GPUs.
For simplicity of exposition, we focus on a single repeating segment: the computation kernels of one nanobatch and the communication kernel of the \emph{previous} nanobatch.
Thus, the communication kernel has no data dependencies with any of the computation kernels in the same segment.
This is the common-case execution pattern throughout training.

\subsubsection{SM Allocation}

First, we study the impact of GPU resource sharing via SM allocation.
Allocating more SMs to communication kernels leaves fewer SMs for overlapped computation kernels given a fixed number of SMs.
Figures~\ref{fig:analysis-vary-sm-2},~\ref{fig:analysis-vary-sm-4}, and~\ref{fig:analysis-vary-sm-20} show execution schedules that vary only in the number of SMs allocated to the communication kernel (2, 4, and 20 SMs, respectively) while launching the communication kernel together with \texttt{Linear 1}.
With two SMs allocated to the communication kernel (Figure~\ref{fig:analysis-vary-sm-2}), the communication does not complete before the computations finish, leaving an exposed communication period during which 106 SMs are idle.
Static power is wasted during this exposed communication time, increasing total energy consumption.
Increasing the number of SMs to four (Figure~\ref{fig:analysis-vary-sm-4}) lets the communication kernel complete before the computation kernels, minimizing static power wastage and total energy consumption.
However, further increasing the number of SMs to 20 (Figure~\ref{fig:analysis-vary-sm-20}) speeds up communication only at the cost of taking SMs away from \texttt{Linear 1}, which ultimately increases total time and energy consumption.
This is because the excess SMs allocated to the communication kernel slow down computation while remaining nearly idle themselves, again wasting static power.

\subsubsection{Communication Launch Timing}

Next, we study the impact of kernels that run together by changing when to launch the communication kernel.
Figure~\ref{fig:analysis-start-norm-1410} shows an execution schedule where the communication kernel runs with four SMs (the same as Figure~\ref{fig:analysis-vary-sm-4}) but starts with \texttt{Norm} instead of \texttt{Linear 1}.
As shown in colors, \texttt{Norm} and \texttt{RoPE} kernels are memory-bound, and with sufficient SMs, the communication kernel also requires high memory bandwidth to write the communicated data to memory.
When such a communication kernel starts together with another memory-bound kernel like \texttt{Norm}, the two compete for memory bandwidth, prolonging both.
This extends the period where compute resources are underutilized, wasting static power.
Furthermore, compared to Figure~\ref{fig:analysis-vary-sm-4}, memory bandwidth is underutilized during \texttt{FlashAttention}, wasting static power in memory components.

\subsubsection{Varying GPU Frequency}

Finally, we study the impact of changing the GPU frequency.
Reducing GPU frequency lowers dynamic energy~\cite{zeus-nsdi23,perseus-sosp24}.
More importantly, the energy-optimal execution schedule changes at lower frequencies.
At lower frequencies, all kernels become relatively more compute-bound because reducing frequency only affects computation throughput and not memory throughput~\cite{nvidia-dvfs-blog}.\footnote{Essentially, the hardware roofline model's horizontal compute-bound ceiling is lowered, making it easier for kernels to be compute-bound.}
Thus, running the communication kernel with a memory-bound kernel causes less interference since memory bandwidth is less contended.
Conversely, running it with a compute-bound kernel and taking SMs away from it causes more severe interference, with larger slowdowns and more static power wastage.

Figure~\ref{fig:analysis-start-norm-1100} is the same as Figure~\ref{fig:analysis-start-norm-1410}, but runs at a lower frequency (1,100 MHz).
Reduced dynamic energy contributes substantially to lower total energy.
However, execution time increases significantly because the communication kernel takes SMs away from \texttt{Linear 1}, which is now even more compute-bound at the lower frequency.
The energy-optimal schedule at this frequency, shown in Figure~\ref{fig:analysis-start-rmsnorm-1100}, launches communication together with \texttt{RoPE}.
This lets the compute-bound \texttt{Linear} kernels run without interference, reducing total execution time and static power wastage.
Figure~\ref{fig:analysis-frontier} plots the time--energy frontier of the six execution schedules.

We also note that exposed communication time is relatively more harmful at lower frequencies.
Dynamic power decreases significantly with frequency while static power does not---this increases static power's proportion in total power draw and thus its relative impact when wasted.

\subsection{Joint Control and Opportunities}\label{sec:analysis-joint-control}

\begin{figure}[t!]
  \begin{center}
    \includegraphics[width=0.3\textwidth]{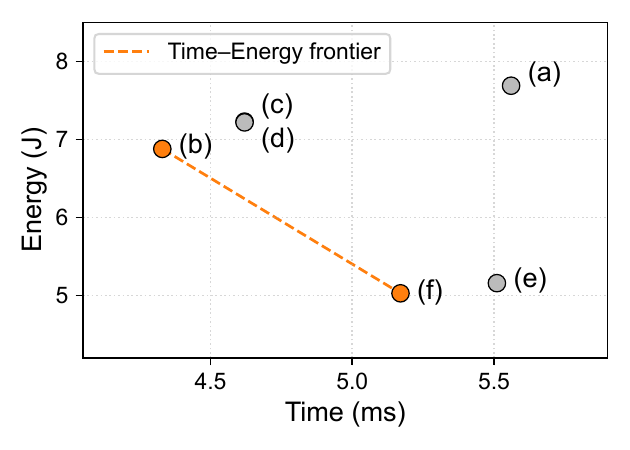}
    \vspace{-0.5em}
    \caption{
      Time and energy consumption of the execution schedules in Figure~\ref{fig:analysis-vary-schedules} (a)--(f).
    }\label{fig:analysis-frontier}
    \vspace{-0.5em}
  \end{center}
\end{figure}

Table~\ref{tab:analysis-observations} summarizes observations from the case studies.
The three factors that determine the execution schedule---SM allocation, communication launch timing, and GPU frequency---are
(1) interdependent: changing one shifts the energy-optimal configuration of the others, and
(2) highly impactful, with up to a 3.29$\times$ gap in time and energy consumption across the observed schedules.
Therefore, optimizing time and energy requires \emph{joint control} of all three.
This is the key insight motivating the design of Kareus's optimization algorithm.

In fact, joint control of GPU frequency and kernel scheduling is provably more energy-efficient than leaving GPU frequency control to the hardware's power controller (Appendix~\ref{apdx:frequency-theorem}).
The intuition is that since dynamic power grows roughly with the cube of frequency,\footnote{Dynamic power is proportional to $V^2 f$. In NVIDIA GPUs, voltage scales roughly linearly with frequency (up to a certain point), so $V^2 f$ becomes $f^3$.} high-frequency periods cost more energy than what low-frequency periods save, so removing frequency variation reduces total dynamic energy.

Our analysis also shows that opportunities for energy savings with execution schedule control are significant.
For simplicity of analysis, we have studied a single Attention layer with modest communication volume (tensor parallelism degree 4).
Even in this setting, without any GPU frequency scaling, the energy-optimal execution schedule (Figure~\ref{fig:analysis-vary-sm-4}) reduces energy consumption by 24.3\% and time by 12.3\% compared to Megatron-LM, and 7.1\% energy and 5.4\% time compared to Nanobatching.
Larger models, more GPUs, and additional parallelism dimensions (e.g., context parallelism) provide even greater opportunities, as we show in Section~\ref{sec:evaluation}.

  \section{Algorithm Design}\label{sec:design}

In this section, we present the optimization algorithm of Kareus, whose goal is to find the time--energy tradeoff frontier of executing computation and communication kernels of large model training.
We begin by formulating the optimization problem (\S\ref{sec:design-formulation}), and then introduce the partitioned overlap execution model, which decomposes the problem into tractable subproblems (\S\ref{sec:design-partition}). 
Building on this model, Kareus efficiently explores the execution schedule space to derive the time--energy frontier for each subproblem (\S\ref{sec:design-bo}), and composes these local frontiers into a global frontier that characterizes the training iteration (\S\ref{sec:design-global}).
Finally, we discuss how Kareus generalizes to broader computation and communication patterns (\S\ref{sec:design-extensions}).

\subsection{Optimization Formulation}\label{sec:design-formulation}

\paragraph{Objective.}
The goal of Kareus is to find an efficient time--energy frontier for the training iteration.
Finding the whole frontier, instead of just the minimum-time point, is valuable because it allows users to select tradeoff points that meet job-level requirements (e.g., time deadlines, energy budgets) or perform dynamic adaptation to changing environments (e.g., stragglers)~\cite{perseus-sosp24,zeus-nsdi23}.
Since the iteration frontier can be constructed by combining \emph{microbatch} (forward and backward) frontiers (\S\ref{sec:design-global}), Kareus's objective reduces to finding microbatch time--energy frontiers that Pareto-dominate those of prior works.

\paragraph{Decision variables.}
We consider three decision variables corresponding to the three key factors that influence time and energy (\S\ref{sec:analysis}):
(1) the number of SMs allocated to communication kernels;
(2) kernel launch timing, which determines the execution order and overlap between kernels while respecting data dependencies; and
(3) GPU frequency.

\paragraph{Solution space.}
Unfortunately, the combined solution space of these three factors is extremely large.
For a typical Transformer-based LLM on an A100 GPU, this space comprises 85K candidate configurations (details in Appendix~\ref{apdx:launch-timing-dp}).
Each candidate in this large search space requires lengthy profiling for accurate energy measurements---on average, 13 seconds per candidate---for thermal stability (\S\ref{sec:implementation}).
An exhaustive search over the entire solution space can take up to 4,912 GPU-hours!

\subsection{Partitioned Overlap Execution Model}\label{sec:design-partition}

\begin{figure}[t!]
  \begin{center}
    \includegraphics[width=0.47\textwidth]{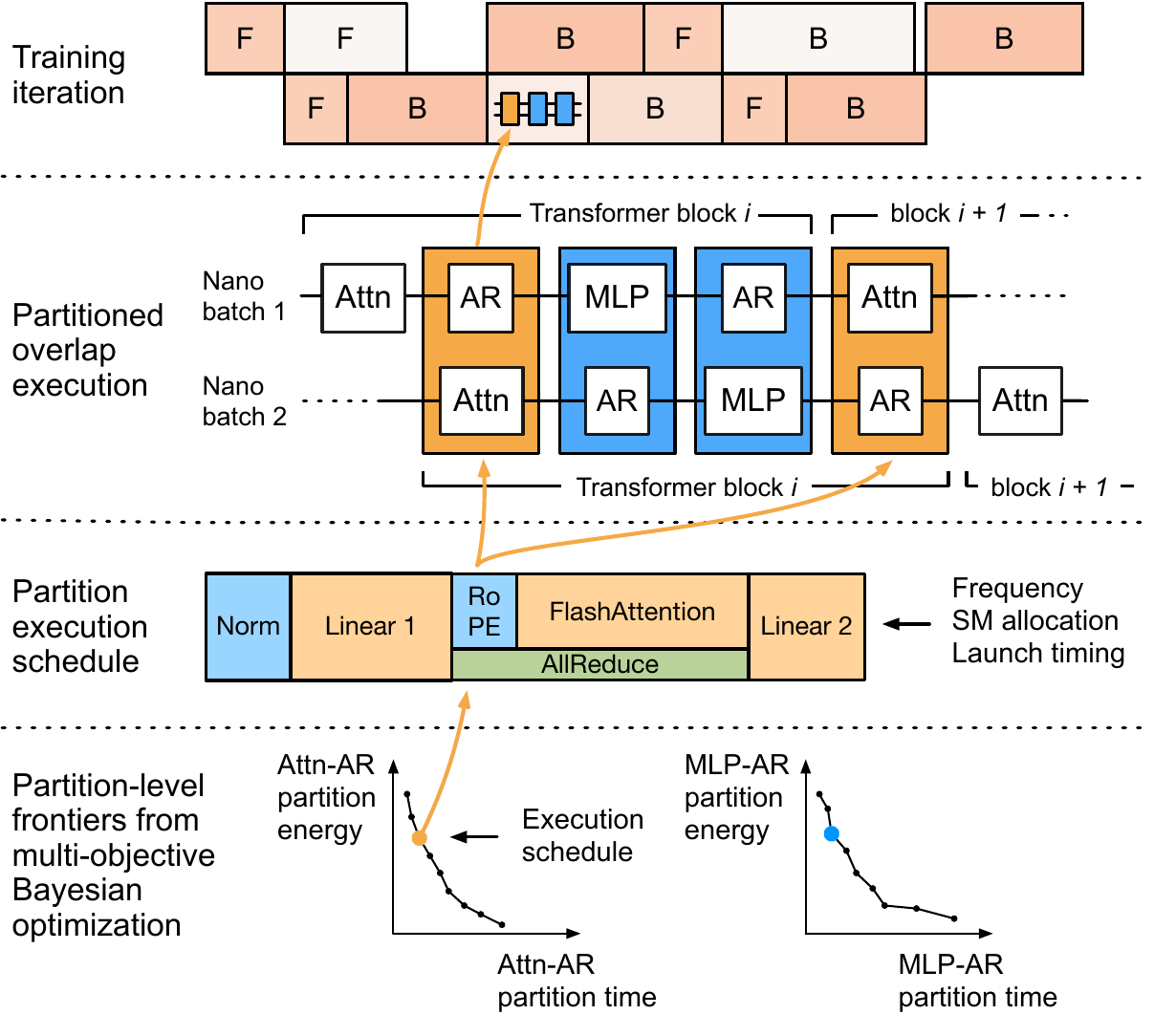}
    \vspace{-0.5em}
    \caption{
      Kareus execution and optimization overview.
      The first row shows one training iteration with the 1F1B pipeline schedule.
      Inside each microbatch, partitions are sequentially executed; the second row illustrates this with the Transformer forward pass with tensor parallelism.
      The third row shows the execution schedule of the Attention--AllReduce partition.
      Each partition's execution schedule is chosen from the partition-level time--energy frontiers (fourth row) characterized by Kareus's optimization algorithm.
    }\label{fig:design-optimization-overview}
    \vspace{-0.5em}
  \end{center}
\end{figure}

To shrink the search space, we identify \emph{repeating patterns} in computation and communication sequences, optimize them independently, and compose them into a global schedule.

\paragraph{Partition.}

Kareus groups kernels executing in repeating patterns into \emph{partitions}.
A partition consists of one communication kernel from one nanobatch and the longest contiguous sequence of computation kernels from the other nanobatch.
There are no data dependencies between the communication kernel and the computation kernels in the same partition.
This allows the communication kernel to overlap with any contiguous subsequence of computation kernels in the partition.

The second row of Figure~\ref{fig:design-optimization-overview} illustrates the repeating pattern of partitions.
The two orange boxes (Attention--AllReduce partition) are identical, as are the two blue boxes (MLP--AllReduce partition), and both partition types repeat across all Transformer blocks.

\paragraph{Optimization overview.}
After automatically detecting the partitions, Kareus's optimizer first characterizes the time--energy frontier of each partition (\S\ref{sec:design-bo}), as shown in the fourth row of Figure~\ref{fig:design-optimization-overview}.
Each point on the frontier represents an efficient execution schedule for that partition, like the one depicted in the third row of Figure~\ref{fig:design-optimization-overview}.
The time--energy frontiers of all partitions are then combined to form the time--energy frontier of each forward and backward microbatch, and then combined again to form the global, iteration-level time--energy frontier (\S\ref{sec:design-global}).

\subsection{Multi-Objective Bayesian Optimization}\label{sec:design-bo}

\subsubsection{Bayesian Optimization Primer}
For one partition, our optimization problem still involves a complex search space with mixed discrete and categorical variables.
Importantly, the time and energy of each configuration have no known closed-form expression and can only be obtained through profiling.
Bayesian Optimization (BO) is well-suited to this setting;
it optimizes objectives over complex search spaces where the objective is unknown in form and costly to evaluate~\cite{bo-tutorial-arxiv18}.

A BO algorithm has two key components.
First, a \emph{surrogate model} is trained on evaluated candidates to approximate the objective function.
It is initialized with random candidates and updated as new candidates are evaluated.
Second, an \emph{acquisition function} uses the surrogate model to rank candidates and select the next one(s) to evaluate.
BO then proceeds iteratively: it selects a batch of candidates using the acquisition function, evaluates them, updates the surrogate model, and repeats until convergence.

\subsubsection{Kareus's MBO Algorithm}
Although BO has been widely applied to systems problems~\cite{cherrypick-nsdi17,autotvm-neurips18,vulcan-nsdi24,baco-asplos23}, it does not fit our setting as is.
Standard BO targets a single objective, whereas we seek the Pareto frontier of time and energy.
We therefore design a Multi-objective Bayesian Optimization (MBO) algorithm with two tailored components: separate surrogate models for time and energy prediction, and a multi-pass candidate selection procedure that explores the frontier from multiple directions.

\paragraph{Surrogate model.}
We train two surrogate models: $\widehat{T}(x)$ for time and $\widehat{E}(x)$ for dynamic energy,\footnote{We use the hat notation to denote predictions, not true measurements.} where $x$ is a candidate configuration.
Since time is primarily influenced by SM allocation and launch timing and dynamic energy by GPU frequency, the two models are largely orthogonal.
Total energy consumption can then be computed as $\widehat{T}(x) \cdot P_{\text{static}} + \widehat{E}(x)$, where $P_{\text{static}}$ is the GPU's static power consumption.

We adopt gradient-boosted decision trees (XGBoost)~\cite{xgboost-kdd16} as the surrogate model for two reasons.
First, XGBoost training scales linearly with data points (versus cubic for Gaussian Processes), allowing fast retraining.
Second, its tree-based structure handles discrete (GPU frequencies and SM allocations) and categorical (launch timing) parameters well.

\paragraph{Multi-pass candidate selection.}
Since our MBO aims to find solutions on the time--energy Pareto frontier rather than a single optimum, we design a \emph{multi-pass} candidate selection procedure.
At each iteration of MBO, we select a batch of $k$ candidates to evaluate next.
The batch consists of candidates selected based on multiple different acquisition functions that are designed to 
(1) cover complementary regions of the time--energy tradeoff space and 
(2) perform both exploration and exploitation to avoid over- or under-exploring certain regions.

\paragraph{Exploitation with hypervolume improvement.}
{\setlength{\columnsep}{14pt}
\begin{wrapfigure}{r}{0.15\textwidth}
  \vspace{-1.2em}
  \begin{center}
    \includegraphics[width=0.17\textwidth]{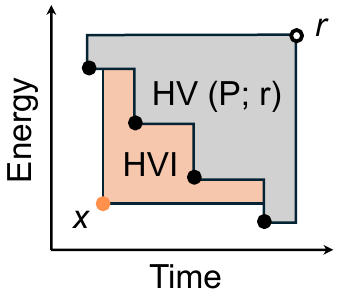}
    \vspace{-1.5em}
    \caption{
      HVI.
    }\label{fig:design-hv}
  \end{center}
\end{wrapfigure}
The acquisition functions for exploitation are based on \emph{hypervolume improvement} (HVI), which quantifies how much a candidate can expand the current time--energy frontier toward lower time and energy (Figure~\ref{fig:design-hv}).
It is defined as
\begin{equation*}
  \mathrm{HVI}(x) = \mathrm{HV}\left(\mathcal{P} \cup \left\{\left(\widehat{T}{(x)}, \mathrm{Energy}(x)\right)\right\}; r\right) - \mathrm{HV}(\mathcal{P}; r)
\end{equation*}
where $\mathcal{P}$ is the current set of candidates on the time--energy frontier, and $r$ is a reference point set slightly worse (outward) than the worst observed points so that the hypervolume (HV) is computable.
By switching the definition of $\mathrm{Energy}(x)$ to total energy, dynamic energy, or static energy, each derived or taken as is from the surrogate models $\widehat{T}(x)$ and $\widehat{E}(x)$, we obtain three HVI acquisition functions.
Each function guides frontier expansion in a distinct direction: dynamic energy favors lower-frequency candidates, static energy favors faster candidates, and total energy lies in the middle.
}

\paragraph{Exploration with uncertainty.}
To reduce the risk of under-exploring certain regions of the search space, we add an uncertainty-based acquisition function.
We quantify uncertainty with \emph{bootstrap ensembles}: multiple surrogate models trained on resampled datasets.
The degree of \emph{disagreement} between the surrogate models serves as a proxy for predictive uncertainty (i.e., how uncertain we are about the prediction).
To promote exploration in either objective, we use the sum of per-objective standard deviations (rather than variances) as the acquisition score to prioritize candidates with high uncertainty.
\looseness=-1

\begin{figure}[t!]
  \begin{center}
    \includegraphics[width=0.40\textwidth]{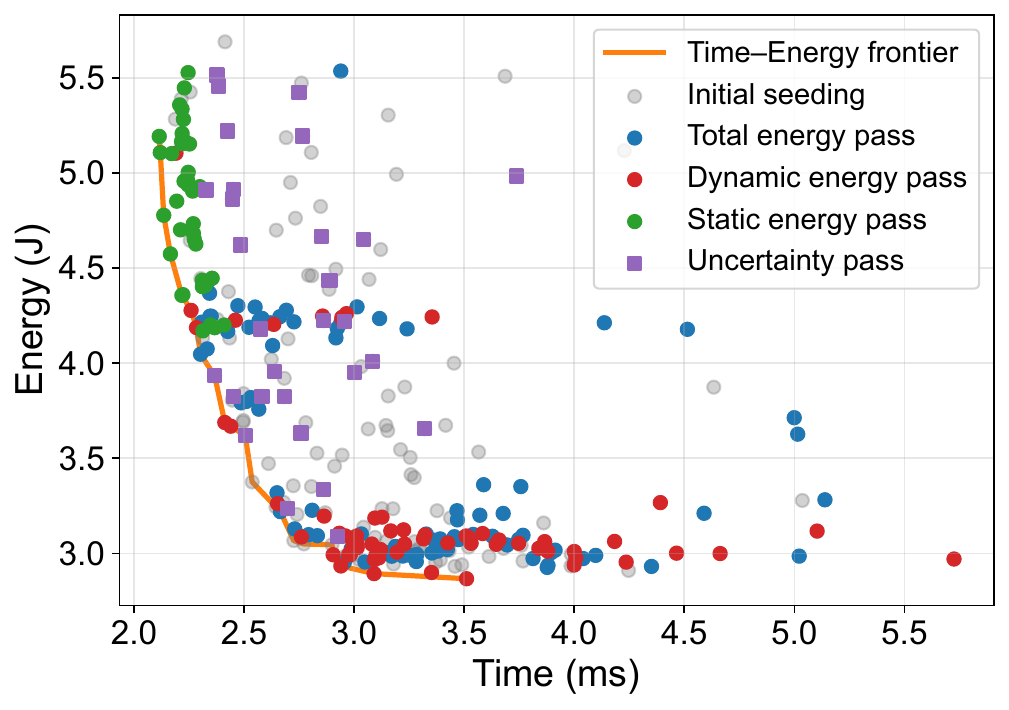}
    \vspace{-0.7em}
    \caption[Multi-pass MBO frontier]{
      Multi-pass MBO in action with the Llama 3.2 3B model's MLP--AllReduce partition.
      The three exploitation passes (total energy, dynamic energy, static energy) and the exploration pass expand the efficient frontier in complementary directions, as shown by the colors of the points that land on the time--energy frontier.
    }\label{fig:design-bo-multi-pass}
    \vspace{-0.5em}
  \end{center}
\end{figure}

\paragraph{The full picture.}
Figure~\ref{fig:design-bo-multi-pass} illustrates frontier expansion for the Llama 3.2 3B model's MLP--AllReduce partition.\footnote{Microbatch size 8, sequence length 4K, and tensor parallelism degree 8.}

Each pass expands the frontier in a complementary direction aligned with its acquisition function: dynamic energy pushes the frontier downward (lower energy), static energy leftward (lower time), and total energy toward the origin.
In contrast, a single acquisition function as used in standard BO is fundamentally insufficient for MBO, as it only targets one direction.
Algorithm~\ref{alg:design-mbo} summarizes the overall Kareus MBO procedure for one partition, and this is repeated for each type of partition independently.
Appendix~\ref{apdx:mbo-hyperparameters} provides detailed settings of hyperparameters including initial random sample size, batch size, ensemble size, and stopping criteria.

\begin{algorithm}[t]
  \caption{
    Multi-pass multi-objective Bayesian Optimization for one partition
  }
  \label{alg:design-mbo}
  \DontPrintSemicolon

  \tcp{Construct initial dataset for surrogate models}
  $\mathcal{D} \leftarrow$ $N_{\text{init}}$ random candidates sampled and evaluated \\
  
  \For{$b = 1,2,\dots,B_{\max}$}{
    Train surrogate models $\widehat{T}(x)$ and $\widehat{E}(x)$ on $\mathcal{D}$. \\

    \tcp{Exploitation: Hypervolume improvement}
    \ForEach{$x \in \mathcal{X} \setminus \mathcal{D}$}{
      Compute $\mathrm{HVI}_{\mathrm{tot}}(x)$, $\mathrm{HVI}_{\mathrm{dyn}}(x)$, $\mathrm{HVI}_{\mathrm{stat}}(x)$.
    }
  
    \tcp{Exploration: Uncertainty via bootstrap ensemble}
    \For{$m = 1$ \KwTo $M$}{
      Train $\widehat{T}^m(x)$ and $\widehat{E}^m(x)$ on resampled $\mathcal{D}$
    }

    \ForEach{$x \in \mathcal{X} \setminus \mathcal{D}$}{
      $\mathrm{Unc}(x) \leftarrow \sigma(\{\widehat{T}^m(x)\}) + \sigma(\{\widehat{E}^m(x)\})$ \\
    }
  
    \tcp{Multi-pass candidate selection}
    $\mathcal{C} \leftarrow \mathrm{TopK}(\{\mathrm{HVI}_{\mathrm{tot}}(x): \forall x \in \mathcal{X} \setminus \mathcal{D}\}, k_1)$ \\
    $\mathcal{C} \leftarrow \mathcal{C} \cup \mathrm{TopK}(\{\mathrm{HVI}_{\mathrm{dyn}}(x): \forall x \in \mathcal{X} \setminus \mathcal{D}\}, k_2)$ \\
    $\mathcal{C} \leftarrow \mathcal{C} \cup \mathrm{TopK}(\{\mathrm{HVI}_{\mathrm{stat}}(x): \forall x \in \mathcal{X} \setminus \mathcal{D}\}, k_3)$ \\
    $\mathcal{C} \leftarrow \mathcal{C} \cup \mathrm{TopK}(\{\mathrm{Unc}(x): \forall x \in \mathcal{X} \setminus \mathcal{D} \setminus \mathcal{C}\}, k - |\mathcal{C}|)$ \\
  
    \tcp{Evaluate candidates and update dataset}
    $\mathcal{D} \leftarrow \mathcal{D} \cup \{(x, T(x), E(x)) : x \in \mathcal{C}\}$ \\
  
    \tcp{Stopping condition}
    $\Delta \leftarrow$ Average HV improvement over $R$ batches \\
    \If{$\Delta < \varepsilon$}{
      \textbf{break}
    }
  }
  \Return{$\mathrm{GetFrontier}(\mathcal{D})$}
  \end{algorithm}

\vspace{0.25em}
\subsection{Composing Frontiers}\label{sec:design-global}

With each partition's time--energy frontier in hand, Kareus composes them into microbatch frontiers, and then into the iteration frontier.

\paragraph{Partition frontiers to microbatch frontier.}
A microbatch consists of multiple partitions executed sequentially.
To construct the microbatch frontier, Kareus enumerates combinations of partition execution schedules (frequency, SM allocation, launch timing), sums their time and energy, and prunes suboptimal combinations.
Algorithm~\ref{alg:microbatch_frontier} illustrates this process.

Two design decisions keep enumeration tractable.
First, Kareus enforces a uniform GPU frequency across all partitions within a microbatch, since frequency switching takes several milliseconds and is non-negligible compared to the latency of one partition.
Second, partitions of the same type share the same SM allocation and launch timing.
For instance, all Attention--AllReduce partitions in Figure~\ref{fig:design-optimization-overview} use identical configurations.
Allowing per-partition configurations would grow the search space exponentially in the number of partitions with little benefit.

  \begin{algorithm}[t]
    \caption{Microbatch frontier construction}
    \label{alg:microbatch_frontier}
    
    $\mathcal{C} \leftarrow \emptyset$  \tcp{All feasible (time, energy) pairs for microbatch}

    \ForEach{$f \in \mathrm{GPU\ frequencies}$}{
      \tcp{Cartesian product of configs across all partitions}
      $\Theta \leftarrow \Pi_{p \in \mathcal{P}} \left( \mathrm{SM\ allocations} \times \mathrm{launch\ timings} \right)_p$ \;
      \ForEach{$\theta \in \Theta$}{
        \tcp{Accumulate time and energy for microbatch}
        $T_m \leftarrow 0, E_m \leftarrow 0$
    
        \ForEach{partition $p \in \mathcal{P}$}{
          $T_m \leftarrow T_m + T_p(f, \theta[p])$ \;
          $E_m \leftarrow E_m + E_p(f, \theta[p])$ \;
        }
    
        \ForEach{non-partition component $c$}{
          $T_m \leftarrow T_m + T_c(f)$ \;
          $E_m \leftarrow E_m + E_c(f)$ \;
        }
    
        $\mathcal{C} \leftarrow
        \mathcal{C} \cup \{(T_m, E_m)\}$
      }
    }
    
    \Return{$\mathrm{GetFrontier}(\mathcal{C})$}
    
    \end{algorithm}

\paragraph{Microbatch frontiers to iteration frontier.}
The first row of Figure~\ref{fig:design-optimization-overview} shows an example of how forward (F) and backward (B) microbatches are scheduled across pipeline stages.
Iteration time is the sum of microbatch latencies along the critical path, while iteration energy combines the energy of all microbatches and the static energy consumed during idle times.
Kareus adopts Perseus's iterative algorithm~\cite{perseus-sosp24} to construct the training iteration frontier from microbatch frontiers.

\subsection{Generalizations}\label{sec:design-extensions}

\paragraph{Multiple communication kernels.}
Section~\ref{sec:design-partition} illustrated partitions using tensor parallelism as a running example.
Real workloads, however, present more complex patterns.
When consecutive communication kernels appear (e.g., multiple AllGather operations under context parallelism to aggregate key and value tensors~\cite{llama3-arxiv24,llama3-isca25}), Kareus fuses them into a single kernel that shares an SM allocation.
This reduces kernel launch overhead and keeps scheduling tractable.
With this, any model and communication scheme reduces to alternating sequences of computations and a single (possibly fused) communication, fitting the partitioned overlap model.

\paragraph{Short consecutive memory-bound computations.}
When multiple short, memory-bound operations appear consecutively (e.g., \texttt{BiasDropoutAdd} followed by \texttt{Norm}), Kareus groups them into one logical operation.
Treating them separately would expand the launch-timing search space with only marginal gains in solution quality.

\paragraph{Execution model switching.}
Partitioned overlap is not universally superior.
When the amount of work within a microbatch is small (e.g., small models or small microbatch sizes), splitting the microbatch can further reduce arithmetic intensity, leading to compute underutilization and higher static power wastage.
In these cases, sequential execution can be more energy-efficient.
To capture this, Kareus also profiles each sequentially executed microbatch at each frontier and includes them as candidates when constructing the microbatch time--energy frontier.
Thus, the final microbatch frontier will automatically be constructed with the better execution model.

\paragraph{GPU power model.}
Kareus adopts a simplified two-component GPU power model that decomposes GPU power into dynamic power and constant static power~\cite{gpuwattch,flipflop-icse26,perseus-sosp24}, which suffices for reasoning about GPU energy consumption in large model training. 
There are finer-grained models that further divide static power into chip leakage and constant power~\cite{accelwattch-micro21,wattchmen-ics26}, where chip leakage is sensitive to voltage.
We leave extending Kareus to such models as future work.

  \section{Implementation}\label{sec:implementation}

\begin{figure}[t!]
    \begin{center}
      \includegraphics[width=0.47\textwidth]{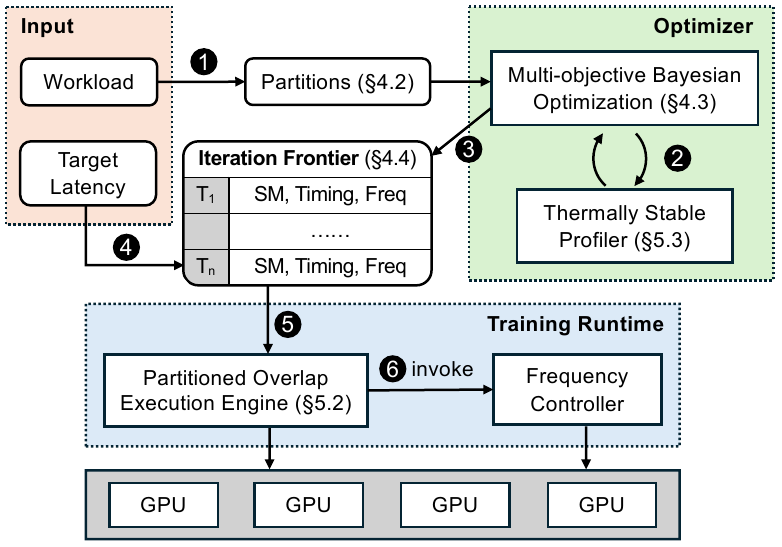}
      \vspace{-0.5em}
      \caption{
        Kareus system overview.
      }\label{fig:system-diagram}
      \vspace{-0.5em}
    \end{center}
  \end{figure}

\subsection{System Overview}\label{sec:implementation-system}

Figure~\ref{fig:system-diagram} shows an overview of the Kareus system. 
Given a workload, Kareus identifies partitions (\blackcircled{1}; \S\ref{sec:design-partition}) and runs multi-objective Bayesian optimization per partition (\blackcircled{2}; \S\ref{sec:design-bo}), evaluating candidates with the thermally stable profiler (\S\ref{sec:implementation-profiling}).
The per-partition frontiers are then composed into the iteration time--energy frontier (\blackcircled{3}; \S\ref{sec:design-global}).
At runtime, given a target iteration latency, Kareus selects an execution schedule from the frontier (\blackcircled{4}) and deploys it to the partitioned overlap execution engine (\blackcircled{5}; \S\ref{sec:implementation-engine}), which invokes the GPU frequency controller to switch each GPU's frequency asynchronously as planned (\blackcircled{6}).

\subsection{Partitioned Overlap Execution Engine}\label{sec:implementation-engine}

Kareus's training execution engine builds on Megatron-LM~\cite{megatronlm-sc21}, executing Transformer blocks as a sequence of partitions.
Before each microbatch execution, Kareus loads the corresponding microbatch configuration, switches between Megatron-LM's default sequential execution and our partitioned overlap execution mode, and configures the SM allocation and launch timing for partitioned overlap execution.
A custom \texttt{torch.autograd.Function} wraps all partitions of a Transformer block, allowing different partitions and partition schedules for forward and backward passes.

Communication kernels in Kareus are implemented with MSCCL++~\cite{mscclpp-github,mscclpp-arxiv25}, which provides fine-grained SM allocation control with grid size.
Computation and communication run on separate CUDA streams to enable overlap, and CUDA events control launch timing.
The GPU frequency controller is adapted from Perseus's open-source implementation~\cite{perseus-sosp24,zeus-github}.

\subsection{Thermally Stable Profiling}\label{sec:implementation-profiling}

Kareus measures the time and energy of a partition during MBO using Zeus~\cite{zeus-nsdi23,zeus-github}, which internally uses NVML~\cite{nvml}.
Accurate energy measurements require care; we present a detailed experimental analysis in Section~\ref{sec:evaluation-profiling}.

\paragraph{Measurement window.}
NVML's sampling interval on NVIDIA GPUs is approximately 100~ms, so millisecond-scale measurements incur large errors.
Kareus executes each partition repeatedly over a 5-second window, after which energy measurements stabilize.

\paragraph{Thermal cooldown.}
The power consumption of any hardware is temperature-dependent.
Without cooldown between candidates, profiles of earlier partitions may heat up the GPU and bias subsequent measurements.
Kareus inserts a 5-second cooldown period, which reliably brings the GPU below 32$^\circ$C in our environment; the required duration depends on the server's cooling capability.
In sum, profiling each candidate takes approximately 13 seconds in our setup, including initialization, warm-up, measurement, and cooldown.

  \section{Evaluation}\label{sec:evaluation}

We evaluate Kareus on 14 workloads and compare it against state-of-the-art baselines.
Our key findings are as follows:

\begin{denseitemize}
  \item Kareus achieves a superior time--energy frontier compared to prior approaches.
In end-to-end training on real GPUs, it delivers up to 28.3\% energy reduction under the same time budget and up to 27.5\% time reduction under the same energy budget compared to the baselines (\S\ref{sec:evaluation-end-to-end}).

  \item In emulated large-scale training, Kareus consistently outperforms the baselines, achieving time and energy reductions comparable to real-world training results (\S\ref{sec:evaluation-emulation}).
  
  \item Ablation and sensitivity analyses show that jointly optimizing dynamic and static energy is necessary (\S\ref{sec:evaluation-ablation}) and is effective across a range of microbatch sizes (\S\ref{sec:evaluation-sensitivity}).

  \item Kareus's MBO algorithm constructs the time--energy frontier with reasonable overhead, and each pass in multi-pass candidate selection is indispensable (\S\ref{sec:evaluation-bo}); the thermally stable profiler delivers stable energy measurements (\S\ref{sec:evaluation-profiling}).
\end{denseitemize}

\subsection{Experimental Setup}\label{sec:evaluation-setup} 

\paragraph{Testbed.}
Experiments are conducted on 16 NVIDIA A100 40GB GPUs deployed across two AWS p4d.24xlarge instances. 
GPUs are fully connected intra-node via NVSwitch, and cross-node bandwidth is 400 Gbps.

\paragraph{Workloads.}
We evaluate Kareus on Llama~3.2~3B~\cite{llama3-arxiv24} and Qwen~3~1.7B~\cite{qwen3-arxiv25} on the physical testbed, and perform large-scale emulation for Llama~3.3~70B~\cite{llama3-arxiv24}.
For pipeline parallelism, we manually partition stages such that stages are as balanced as possible, following Perseus~\cite{perseus-sosp24}.
For context parallelism, we follow the scheme used in Llama 3~\cite{llama3-arxiv24}, where key--value tensors are collected across GPUs via AllGather.
We use activation checkpointing to reduce memory pressure.

\paragraph{Baselines.}
We mainly compare against three baselines:
\begin{denseitemize}
  \item \textbf{Megatron-LM (``M'').}
    Baseline Megatron-LM~\cite{megatronlm-sc21} with the sequential execution model and maximum GPU frequency.
    Produces a single point on the time--energy plane.

  \item \textbf{Megatron-LM + Perseus (``M+P'').}
    Perseus~\cite{perseus-sosp24} applied to the above.
    This produces a time--energy frontier.

  \item \textbf{Nanobatching + Perseus (``N+P'').}
    A training engine based on Megatron-LM that implements nanobatching, a special case of Kareus's partitioned overlap execution model (\S\ref{sec:analysis-case-study}), integrated with Perseus like the above.
    This also produces a time--energy frontier.
\end{denseitemize}

\paragraph{Metrics.}
Kareus's gain is fundamentally a superior time--energy tradeoff frontier, which cannot be fully captured by a single number.
Therefore, we define two modes of comparison based on two viable use cases. %

\begin{denseitemize}
  \item \textbf{Max-throughput comparison.}
  When iteration time constraints (e.g., deadlines, stragglers) are not present, the training pipeline operates in maximum-throughput mode.
  For methods that produce a time--energy frontier, this means operating at the leftmost (lowest time) point.
  We set Megatron-LM as the baseline, and report iteration time and energy reduction (\%) of M+P, N+P, and Kareus.

  \item \textbf{Frontier improvement.}
  We set M+P as the baseline, and define two metrics that quantify frontier improvements for N+P and Kareus (Figure~\ref{fig:eval-frontier-example}):
  (1) \emph{Iso-time energy reduction} (\%): energy reduction with iteration time deadline set as M+P's minimum iteration time; and
  (2) \emph{Iso-energy time reduction} (\%): time reduction with iteration energy budget set as M+P's minimum iteration energy.

\end{denseitemize}

\begin{figure}[t!]
    \begin{center}
      \includegraphics[width=0.30\textwidth]{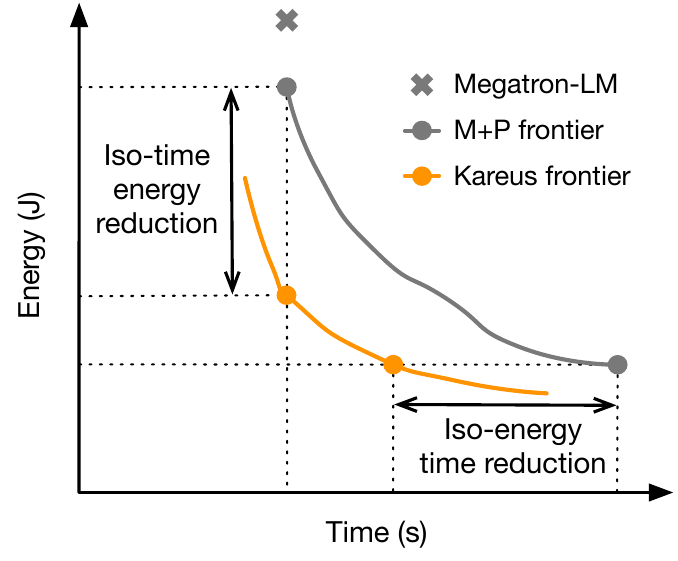}
      \vspace{-1em}
      \caption{Illustrative example showing iteration time--energy frontiers. For the max-throughput comparison, we compare the leftmost points of each frontier with Megatron-LM. For frontier improvement, we report the iso-time energy reduction and iso-energy time reduction compared to Megatron-LM + Perseus.
      }\label{fig:eval-frontier-example}
      \vspace{-0.5em}
    \end{center}
  \end{figure}

\subsection{End-to-End Results}\label{sec:evaluation-end-to-end}

\begin{table*}[t]
\centering
\caption{
  [Experiment] Iteration time and energy reductions (\%) relative to Megatron-LM (higher is better) for all methods under the max-throughput configuration.
  Pipeline parallelism degree is fixed at 2, and number of microbatches is fixed at 8.
  \texttt{OOM} indicates that the GPU runs out of memory at the corresponding settings.
  Negative values indicate increased time or energy relative to the baseline.
}\label{tab:e2e-leftmost}
\vspace{-0.5em}
\resizebox{0.94\textwidth}{!}{%
\begin{tabular}{llrr|rrrrrr|r}
    \toprule
    \multicolumn{1}{c}{\multirow{3}{*}{\textbf{Model}}} &
    \multicolumn{1}{c}{\multirow{3}{*}{\textbf{Parallelism}}} &
    \multicolumn{1}{c}{\multirow{3}{*}{\textbf{\begin{tabular}[c]{@{}c@{}}ubatch\\Size\end{tabular}}}} &
    \multicolumn{1}{c}{\multirow{3}{*}{\textbf{\begin{tabular}[c]{@{}c@{}}Sequence\\Length\end{tabular}}}} &
    \multicolumn{3}{c}{\textbf{Time Reduction (\%)}} &
    \multicolumn{3}{c}{\textbf{Energy Reduction (\%)}} &
    \multicolumn{1}{c}{\multirow{3}{*}{\textbf{\begin{tabular}[c]{@{}c@{}}TFLOP/s\\/GPU\end{tabular}}}} \\
    \cmidrule(lr){5-7}\cmidrule(lr){8-10}
     &  &  &  &
    \multicolumn{1}{c}{Megatron-LM} &
    \multicolumn{1}{c}{Nanobatching} &
    \multicolumn{1}{c}{\multirow{2}{*}{Kareus}} &
    \multicolumn{1}{c}{Megatron-LM} &
    \multicolumn{1}{c}{Nanobatching} &
    \multicolumn{1}{c|}{\multirow{2}{*}{Kareus}} & \\
     &  &  &  & \multicolumn{1}{c}{+Perseus} & \multicolumn{1}{c}{+Perseus} &  & \multicolumn{1}{c}{+Perseus} & \multicolumn{1}{c}{+Perseus} & \multicolumn{1}{c|}{} & \\
    \midrule
    \multirow{3}{*}{Llama 3.2 3B} & \multirow{3}{*}{TP8} & 8 & 4K & $-0.3$ & 8.5 & 12.3 & 10.0 & 15.5 & 19.6 & 104.8 \\
     &  & 8 & 8K & \multicolumn{6}{c|}{OOM} & \multicolumn{1}{c}{} \\
     &  & 16 & 4K & \multicolumn{6}{c|}{OOM} & \multicolumn{1}{c}{} \\
    \cmidrule(lr){2-11}
    \multirow{3}{*}{Llama 3.2 3B} & \multirow{3}{*}{CP2TP4} & 8 & 4K & 0.2 & 0.0 & 5.2 & 7.4 & 7.3 & 14.4 & 117.3 \\
     &  & 8 & 8K & $-1.4$ & 1.5 & 6.2 & 11.5 & 7.3 & 14.9 & 138.0 \\
     &  & 16 & 4K & $-0.7$ & 2.4 & 8.0 & 8.9 & 7.0 & 16.2 & 131.0 \\
    \cmidrule(lr){1-11}
    \multirow{3}{*}{Qwen 3 1.7B} & \multirow{3}{*}{TP8} & 8 & 4K & $-0.5$ & 5.6 & 12.2 & 7.7 & 12.8 & 22.1 & 88.5 \\
     &  & 8 & 8K & 0.1 & 9.0 & 14.9 & 7.9 & 12.2 & 15.1 & 107.2 \\
     &  & 16 & 4K & $-0.1$ & 9.4 & 14.8 & 7.6 & 13.4 & 17.3 & 97.0 \\
    \cmidrule(lr){2-11}
    \multirow{3}{*}{Qwen 3 1.7B} & \multirow{3}{*}{CP2TP4} & 8 & 4K & $-0.5$ & $-20.4$ & $-0.5$ & 7.1 & 3.1 & 7.1 & 88.9 \\
     &  & 8 & 8K & 0.5 & 1.7 & 8.0 & 7.4 & 6.7 & 8.7 & 119.4 \\
     &  & 16 & 4K & 0.0 & 3.9 & 10.5 & 6.9 & 6.9 & 11.2 & 110.6 \\
    \bottomrule
    \end{tabular}%
}
\end{table*}

In this section, we evaluate the iteration time and energy performance on the testbed GPUs for the max-throughput (\S\ref{sec:eval-e2e-leftmost}) and frontier improvement (\S\ref{sec:eval-e2e-frontier}) comparisons.
The workload configurations (parallelism, microbatch size, and sequence length) can be found in Table~\ref{tab:e2e-leftmost}.

\subsubsection{Max-Throughput Comparison}\label{sec:eval-e2e-leftmost}

Table~\ref{tab:e2e-leftmost} reports the iteration time and energy reductions achieved by all methods under the max-throughput regime.  
Overall, Kareus achieves up to 14.9\% reduction in iteration time and 22.1\% reduction in energy consumption compared to Megatron-LM, strictly outperforming the baselines on time and energy.
\looseness=-1

\paragraph{Time reduction.}
The time reduction mainly comes from reducing SM idle time caused by resource underutilization due to exposed communication or resource contention by poorly scheduled kernel overlaps.
As a result, the time reduction trend largely follows the communication overhead induced by different model configurations. 
Compared to Megatron-LM, both overlap execution models achieve larger gains under tensor parallelism (TP) than under context parallelism with tensor parallelism (CP+TP).
Moreover, increasing the microbatch size yields greater time reduction than increasing the sequence length, since the former introduces relatively higher communication overhead.

When compared to Nanobatching, Kareus provides larger additional gains under CP+TP than under TP alone, because fine-grained control of communication scheduling becomes more critical for complex communication patterns. 
Moreover, Nanobatching can significantly slow down computation (e.g., on Qwen~1.7B with CP2+TP4, microbatch size 8, and sequence length 4K), as it leads to GPU underutilization when the workload is small. In contrast, Kareus can automatically switch to the sequential execution model (\S\ref{sec:design-extensions}).

\paragraph{Energy reduction.}
Energy reflects the combined effects of time and power, which is why energy reduction trends do not always track time reduction trends.
In many cases, Nanobatching + Perseus reduces time but yields smaller energy reduction than Megatron-LM + Perseus, because overlap raises GPU utilization and thus also power, and the power increase offsets time reduction.
In contrast, Kareus achieves larger energy reductions by simultaneously optimizing overlap and GPU frequency, enabling greater time reduction while also identifying opportunities to reduce power, as the following case studies show.

\paragraph{Case studies.}
We find an interesting case in Qwen 3 1.7B with TP8, microbatch size 8, and sequence length 4K.
Both baselines operate the critical-path microbatch at 1,410 MHz (the maximum GPU frequency).
However, Kareus discovers an overlap configuration that runs at 1,350 MHz yet attains the fastest latency, resulting in an additional 10.7\% energy reduction over Nanobatching + Perseus.
This happens because nanobatching overlap at 1,410 MHz improves resource utilization but raises instantaneous power, which triggers GPU frequency throttling.
Thus, the time-averaged GPU frequency (which determines execution time) is close to 1,350 MHz, while the time-averaged dynamic power (which determines dynamic energy) remains closer to that of 1,410 MHz.
Operating at a constant frequency provably consumes less energy than operating at a fluctuating frequency with the same average, which highlights the importance of jointly optimizing kernel scheduling and GPU frequency scaling (\S\ref{sec:analysis-joint-control}).

Figure~\ref{fig:case-study} further illustrates representative partition execution schedules across microbatches of the aforementioned configuration.
In the forward and backward Attention--AllReduce partitions, Kareus decides not to overlap \texttt{AllReduce} with memory-bound kernels (e.g., \texttt{Norm}) at a higher frequency (1,350 MHz); instead, it shifts the overlap to more memory-bound kernels at a lower frequency (1,290 MHz), consistent with the analysis in Section~\ref{sec:analysis}.
For the MLP--AllReduce partitions, the selected 
configurations differ due to variations in kernel lengths and 
types.
Kareus starts overlap from the beginning and allocates only 6 SMs to \texttt{AllReduce} at 1,350 MHz, which provides enough time to hide communication while leaving most SMs available for the long compute-bound \texttt{Linear 1} kernels.

\begin{figure*}[t]
    \centering
    \includegraphics[width=0.85\textwidth]{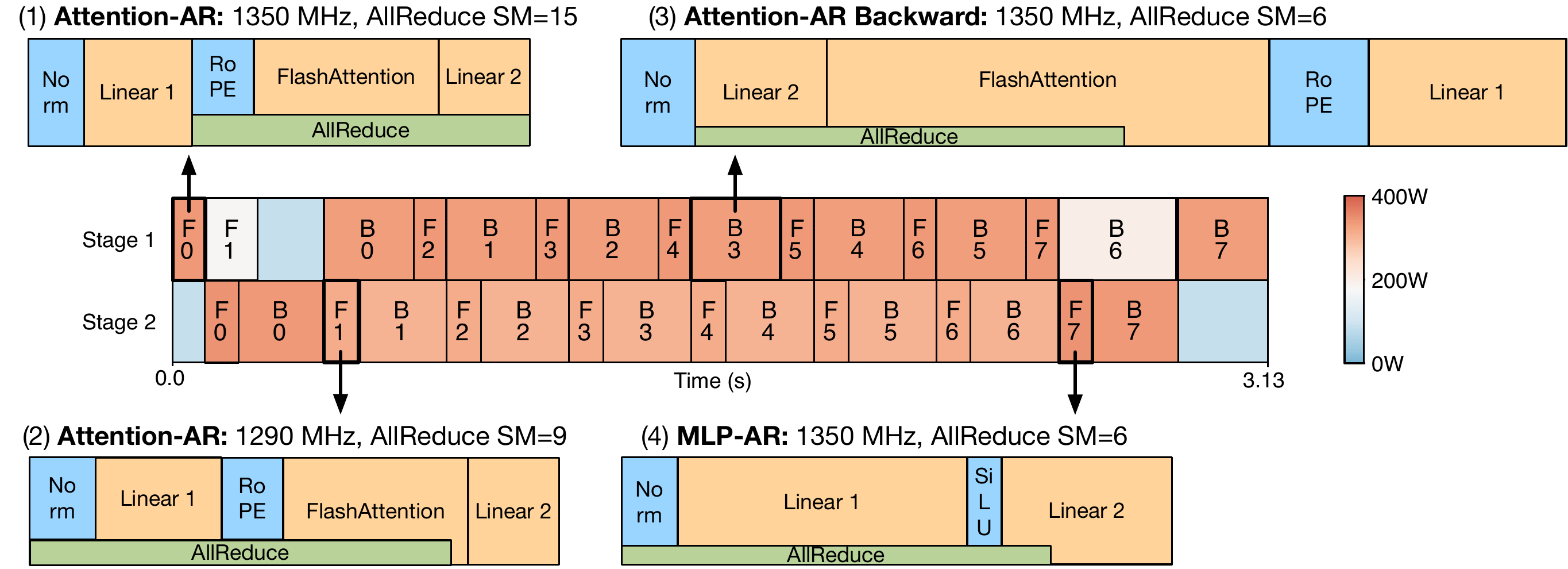}
    \vspace{-0.5em}
    \caption{Case study of Qwen 3 1.7B with TP8, microbatch size 8, and sequence length 4K. The pipeline schedule is colored by average GPU power draw, and the partition execution schedules are drawn to scale for time. (1) and (2) compare the forward Attention--AllReduce partition at 1,350 MHz and 1,290 MHz, respectively. (3) and (4) show the backward Attention--AllReduce and forward MLP--AllReduce partitions at 1,350 MHz. Backward microbatch is partitioned following the definition in Section~\ref{sec:design-partition}; \texttt{Norm} is treated as the first kernel because it follows the AllReduce kernel in the transformer block, and other kernels are ordered in reverse.}\label{fig:case-study}
    \vspace{-1em}
\end{figure*}

\subsubsection{Frontier Improvement}\label{sec:eval-e2e-frontier}

Table~\ref{tab:e2e-iso} reports the iso-time energy reduction and iso-energy time reduction for Nanobatching + Perseus and Kareus, which demonstrate the expansion of the time--energy frontier.
Overall, Kareus achieves up to 28.3\% iso-time energy reduction and 27.5\% iso-energy time reduction compared to Megatron-LM + Perseus, establishing a strictly better time--energy tradeoff over Nanobatching + Perseus.
Figure~\ref{fig:eval-frontier-real} shows a representative frontier comparison for Qwen 3 1.7B with CP2, TP4, microbatch size 16, and sequence length 4K.
We present the complete frontier comparison plots for all model configurations in Appendix~\ref{apdx:frontier-e2e}.

\begin{figure}[t]
    \centering
    \includegraphics[width=0.9\columnwidth]{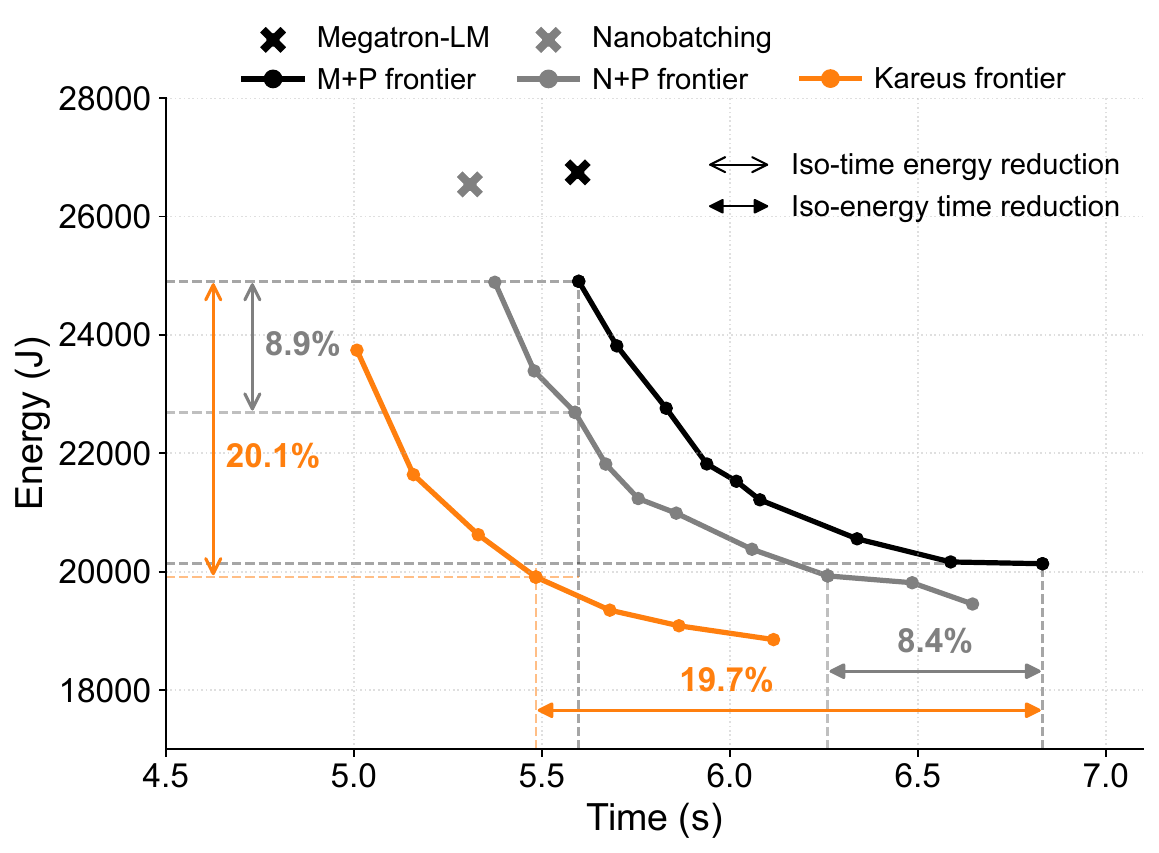}
    \vspace{-0.5em}
    \caption{[Experiment] Representative iteration time--energy frontier comparison for Qwen 3 1.7B with CP2, TP4, microbatch size 16, and sequence length 4K.}\label{fig:eval-frontier-real}
\end{figure}

Under the same time budget, where static energy is fixed, Kareus achieves higher GPU utilization and can therefore execute at lower GPU frequencies to complete the workload, reducing dynamic energy consumption. Conversely, under the same energy budget, Kareus minimizes static energy wastage, allowing more energy to be allocated to dynamic execution at higher GPU frequencies to complete the workload faster.

\begin{table}[t]
    \centering
    \caption{[Experiment] Iso-time energy reduction (\%) and iso-energy time reduction (\%) relative to Megatron-LM + Perseus for Nanobatching + Perseus and Kareus. "$-$" indicates that no configuration satisfies that constraint. For example, if Nanobatching + Perseus has a longer iteration time than Megatron-LM + Perseus at the leftmost point, then no iso-time point exists.}\label{tab:e2e-iso}
    \vspace{-0.5em}
\resizebox{0.48\textwidth}{!}{%
\begin{tabular}{llrr|rrrr}
    \toprule
    \multicolumn{4}{c}{} & \multicolumn{2}{c}{\textbf{\shortstack{Iso-Time Energy\\Reduction (\%)}}} & \multicolumn{2}{c}{\textbf{\shortstack{Iso-Energy Time\\Reduction (\%)}}} \\
    \cmidrule(lr){5-6}\cmidrule(lr){7-8}
    \multicolumn{1}{c}{} & \multicolumn{1}{c}{} & \multicolumn{1}{c}{} & \multicolumn{1}{c}{} & N+P & Kareus & N+P & Kareus \\
    \midrule
    \multirow{3}{*}{3B} & \multirow{3}{*}{TP8} & 8 & 4K & 21.0 & 24.3 & 18.9 & 24.0 \\
     &  & 8 & 8K & \multicolumn{4}{c}{OOM} \\
     &  & 16 & 4K & \multicolumn{4}{c}{OOM} \\
    \cmidrule(lr){2-8}
    \multirow{3}{*}{3B} & \multirow{3}{*}{CP2TP4} & 8 & 4K & $-$ & 17.4 & 7.0 & 13.9 \\
     &  & 8 & 8K & 4.0 & 12.4 & 4.4 & 16.4 \\
     &  & 16 & 4K & 6.6 & 20.4 & 5.1 & 12.5 \\
    \cmidrule(lr){1-8}
    \multirow{3}{*}{1.7B} & \multirow{3}{*}{TP8} & 8 & 4K & 16.8 & 26.8 & 17.8 & 27.5 \\
     &  & 8 & 8K & 20.0 & 23.1 & 17.7 & 23.1 \\
     &  & 16 & 4K & 20.4 & 28.3 & 19.9 & 26.7 \\
    \cmidrule(lr){2-8}
    \multirow{3}{*}{1.7B} & \multirow{3}{*}{CP2TP4} & 8 & 4K & $-$ & 0.0 & $-$ & 5.1 \\
     &  & 8 & 8K & $-0.7$ & 15.0 & $-0.4$ & 12.1 \\
     &  & 16 & 4K & 8.9 & 20.1 & 8.4 & 19.7 \\
    \bottomrule
    \end{tabular}
}
\vspace{0.3em}
\end{table}

\subsection{Large-Scale Emulation}\label{sec:evaluation-emulation}
\vspace{-0.5em}

In this section, we conduct large-scale emulation of Llama 3.3 70B based on smaller-scale profiling on our testbed hardware and compare against Megatron-LM + Perseus.
We use the emulator from Perseus~\cite{perseus-sosp24}.
Results show that trends in time and energy reduction are consistent with those in Section~\ref{sec:evaluation-end-to-end} measured on testbed GPUs.

\paragraph{Methodology and parameters.}
For Kareus, we obtain the iteration-level time--energy frontier using the MBO algorithm described in Section~\ref{sec:design} by composing partition-level frontiers. 
For Megatron-LM + Perseus, we profile the time and energy of each Transformer block using the same profiling methodology described in Section~\ref{sec:implementation-profiling}, and construct the iteration-level frontier.
We perform strong scaling as we vary the number of GPUs (Table~\ref{tab:emulation-configs}), while fixing the global batch size to 2048, which is adopted in the Llama 3 training~\cite{llama3-arxiv24}.
We use pipeline parallelism degree 10 and tensor parallelism degree 8, with microbatch size 4 and sequence length 4K~\cite{llama3-arxiv24}.

\begin{table}[t]
    \centering
    \caption{Strong scaling configurations for large-scale emulation.}\label{tab:emulation-configs}
    \vspace{-0.5em}
    \resizebox{0.42\textwidth}{!}{%
    \begin{tabular}{rrrr}
      \toprule
      \multicolumn{1}{r}{\multirow{2}{*}{\textbf{\# GPUs}}} &
      \multicolumn{1}{r}{\multirow{2}{*}{\textbf{\# Pipelines}}} &
      \multicolumn{1}{r}{\textbf{\# Microbatches}} &
      \multicolumn{1}{r}{\multirow{2}{*}{\textbf{Global Batch Size}}} \\
       &  & \multicolumn{1}{r}{\textbf{Per Pipeline}} &  \\
      \midrule
       10240 & 128 & 16  & \multirow{4}{*}{2048} \\
       5120  & 64  & 32  &  \\
       2560  & 32  & 64  &  \\
       1280  & 16  & 128 &  \\
      \bottomrule
    \end{tabular}
    }
  \end{table}

\paragraph{Max-throughput comparison.}
We report the max-throughput comparison results for Megatron-LM + Perseus and Kareus across different numbers of microbatches in Table~\ref{tab:emulation-leftmost}, and plot the corresponding iteration-level time--energy frontiers in Appendix~\ref{apdx:frontier-emulation}. 
The energy reductions achieved by both methods are generally higher than those observed in the testbed experiments, because the larger number of pipeline stages and microbatches on non-critical paths can be slowed down to save dynamic energy.
As the number of microbatches increases, the relative portion of pipeline bubbles during the warm-up and cooldown phases---which are normally reduced down to the lowest frequency---decreases, leading to a slight decrease in energy reduction.
The time reduction achieved by Kareus is slightly lower than that of the TP=8 case in the testbed experiments, because a smaller microbatch size is used in this setting due to the GPU memory constraints.

\begin{table}[t]
    \centering
    \caption{[Emulation] Iteration time and energy reduction (\%) relative to Megatron-LM of Megatron-LM + Perseus and Kareus under the max-throughput configuration for Llama 3.3 70B.}\label{tab:emulation-leftmost}
    \vspace{-0.5em}
    \resizebox{0.46\textwidth}{!}{%
\begin{tabular}{r|rrrr}
    \toprule
    \multirow{2}{*}{\textbf{\# Microbatches}} & \multicolumn{2}{c}{\textbf{Time Reduction (\%)}} & \multicolumn{2}{c}{\textbf{Energy Reduction (\%)}} \\
    \cmidrule(lr){2-3}\cmidrule(lr){4-5}
     & M+P & Kareus & M+P & Kareus \\
    \midrule
    16 & 0.0 & 9.3 & 15.0 & 20.2 \\
    32 & 0.0 & 9.2 & 14.3 & 20.0 \\
    64 & 0.0 & 9.1 & 13.8 & 19.8 \\
    128 & 0.0 & 9.1 & 13.5 & 19.7 \\
    \bottomrule
    \end{tabular}
    }
    \vspace{0.3em}
\end{table}

\paragraph{Frontier improvement.}
We report the frontier improvement results for Kareus across different numbers of microbatches in Table~\ref{tab:emulation-iso}.
As the number of microbatches increases, the iso-time energy reduction improves, because the absolute iteration time grows, providing more slack for dynamic energy savings.
In contrast, the iso-energy time reduction decreases, since static energy consumes a larger fraction of the fixed energy budget.

\begin{table}[t]
    \centering
    \caption{[Emulation] Iso-time energy reduction (\%) and iso-energy time reduction (\%) relative to Megatron-LM + Perseus of Kareus for Llama 3.3 70B.
    }\label{tab:emulation-iso}
    \vspace{-0.5em}
    \resizebox{0.43\textwidth}{!}{%
    \begin{tabular}{l|rrrr}
        \toprule
         & \multicolumn{4}{c}{\textbf{\# Microbatches}} \\
         & 16 & 32 & 64 & 128 \\
        \midrule
        \textbf{Iso-Time Energy Reduction (\%)} & 11.6 & 14.1 & 15.3 & 15.1 \\
        \textbf{Iso-Energy Time Reduction (\%)} & 19.1 & 16.8 & 16.4 & 16.0 \\
        \bottomrule
        \end{tabular}
}
\vspace{0.3em}
\end{table}

\subsection{Ablation Study on Search Space}\label{sec:evaluation-ablation}

In this section, we conduct an ablation study to demonstrate the necessity of jointly optimizing dynamic and static energy in Kareus. 
We compare Kareus against three ablated variants: removing dynamic energy optimization (i.e., frequency scaling), removing static energy optimization (i.e., kernel scheduling with SM allocation and launch timing), and removing both (i.e., Nanobatching). 
We evaluate all systems on Qwen 3 1.7B with PP=2, TP=8, 8 microbatches of size 8, and sequence length 4K, the same configuration used in end-to-end experiments (Tables~\ref{tab:e2e-leftmost} and~\ref{tab:e2e-iso}).

Table~\ref{tab:ablation} reports the iteration time and energy increases of each ablated variant relative to Kareus under the max-throughput configuration.
Removing frequency scaling results in a 12.9\% increase in energy, and removing kernel scheduling leads to a 10.8\% increase in energy.
As expected, removing either of the optimization dimensions fails to deliver the full optimization potential of Kareus, demonstrating the necessity of joint optimization of both dynamic and static energy.

\begin{table}[t]
    \centering
    \caption{[Experiment] Iteration time and energy increase (\%) of ablated variants relative to Kareus under the max-throughput configuration.}\label{tab:ablation}
    \vspace{-0.5em}
    \resizebox{0.45\textwidth}{!}{%
\begin{tabular}{l|rr}
    \toprule
    \multicolumn{1}{c|}{\textbf{System}} &
    \multicolumn{1}{c}{\textbf{Time Inc. (\%)}} &
    \multicolumn{1}{c}{\textbf{Energy Inc. (\%)}} \\
    \midrule
    Kareus w/o frequency & 1.0 & 12.9 \\
    Kareus w/o kernel schedule & 8.2 & 10.8 \\
    Nanobatching & 7.8 & 20.6 \\
    \bottomrule
    \end{tabular}
    }
    \vspace{0.3em}
\end{table}

\subsection{Sensitivity to Microbatch Size}\label{sec:evaluation-sensitivity}

In this section, we further study how the microbatch size, an important training parameter that directly affects computation arithmetic intensity, the ratio of communication to computation, and pipeline imbalance, affects the effectiveness of Kareus.
We choose Qwen 3 1.7B with TP=8 and sequence length 4K as in Section~\ref{sec:evaluation-ablation} and vary the microbatch size from 8 to 20.\footnote{Larger microbatch sizes are not evaluated due to GPU memory capacity.}
Tables~\ref{tab:sensitivity-leftmost} and~\ref{tab:sensitivity-iso} report the max-throughput comparison and frontier improvement results, respectively; Appendix~\ref{apdx:frontier-sensitivity} shows the corresponding time--energy frontiers.

Overall, Kareus demonstrates consistent effectiveness across varying microbatch sizes.
As microbatch size increases, Kareus achieves greater time reduction (Table~\ref{tab:sensitivity-leftmost}), as communication--computation overlap more effectively utilizes SMs when decomposing microbatches into nanobatches.
This is also why microbatch size 20 achieves the best iso-time and iso-energy reductions (Table~\ref{tab:sensitivity-iso}).
Max-throughput energy reduction (Table~\ref{tab:sensitivity-leftmost}) shows a more complex trend for two reasons.
First, microbatch size 8 has a pronounced energy reduction due to the same reason as the case study in Section~\ref{sec:eval-e2e-leftmost}, where Kareus discovers a more energy-efficient critical path execution schedule that runs at a lower frequency.
Second, changing the microbatch size impacts pipeline stage imbalance, which affects the potential for energy savings by slowing down non-critical microbatches---this worked in favor of microbatch size 12 and 20 by increasing imbalance and allowing more energy savings.

\begin{table}[t]
    \centering
    \caption{[Experiment] Iteration time and energy reduction (\%) relative to Megatron-LM of Megatron-LM + Perseus and Kareus across different microbatch sizes.}\label{tab:sensitivity-leftmost}
    \vspace{-0.5em}
    \resizebox{0.43\textwidth}{!}{%
\begin{tabular}{r|rrrr}
    \toprule
    \multirow{2}{*}{\textbf{ubatch Size}} & \multicolumn{2}{c}{\textbf{Time Reduction (\%)}} & \multicolumn{2}{c}{\textbf{Energy Reduction (\%)}} \\
    \cmidrule(lr){2-3}\cmidrule(lr){4-5}
     & M+P & Kareus & M+P & Kareus \\
    \midrule
    8 & -0.5 & 12.2 & 7.7 & 22.1 \\
    12 & -0.4 & 14.7 & 9.4 & 17.6 \\
    16 & -0.1 & 14.8 & 7.6 & 17.3 \\
    20 & -0.2 & 18.1 & 9.4 & 18.8 \\
    \bottomrule
    \end{tabular}
    }
    \vspace{0.3em}
\end{table}

\begin{table}[t]
    \centering
    \caption{[Experiment] Iso-time energy reduction (\%) and iso-energy time reduction (\%) relative to Megatron-LM + Perseus of Kareus across different microbatch sizes.}\label{tab:sensitivity-iso}
    \vspace{-0.5em}
    \resizebox{0.43\textwidth}{!}{%
    \begin{tabular}{l|rrrr}
        \toprule
         & \multicolumn{4}{c}{\textbf{Microbatch Size}} \\
         & 8 & 12 & 16 & 20 \\
        \midrule
        \textbf{Iso-Time Energy Reduction (\%)} & 26.8 & 28.6 & 28.3 & 29.8 \\
        \textbf{Iso-Energy Time Reduction (\%)} & 27.5 & 27.3 & 26.7 & 28.8 \\
        \bottomrule
        \end{tabular}
}
\vspace{0.3em}
\end{table}

\subsection{MBO Analysis}\label{sec:evaluation-bo}

\paragraph{Overhead and breakdown.}
Kareus's MBO (\S\ref{sec:design-bo}) obtains the time--energy frontier for each partition.
In our testbed experiments, MBO for different partitions is conducted in parallel and takes 2 hours (32 GPU hours) on average. 
This overhead is substantially lower than that of exhaustive search (307 hours, as discussed in Section~\ref{sec:design-formulation}), and is negligible compared to typical end-to-end training times (e.g., 54 days for Llama 3~\cite{llama3-arxiv24}).

The overhead of MBO is dominated by thermally stable profiling of candidates (\S\ref{sec:implementation-profiling}), accounting for 97\% of the total overhead.
For a typical partition in our setting, which evaluates a batch of 32 candidates per MBO iteration, thermally stable profiling takes 6.9 minutes on average, while surrogate model training and acquisition function evaluation takes 11 seconds on average.

\paragraph{Multi-pass candidate selection.}
After random candidate initialization, MBO selects candidates across four passes: total energy pass, dynamic energy pass, static energy pass, and uncertainty pass (\S\ref{sec:design-bo}).
On average, random initialization contributes 35\% of the selected candidates---they happen to be on the frontier.
The other 65\% of the selected candidates are discovered by the total energy pass, dynamic energy pass, static energy pass, and uncertainty pass, each 25\%, 22\%, 12\%, and 6\% of the selected candidates, respectively.
All passes contribute a non-negligible portion of the final candidates.

\subsection{Thermally Stable Profiler}
\label{sec:evaluation-profiling}

During MBO, Kareus adopts thermally stable profiling to measure the time and energy of each candidate (\S\ref{sec:implementation-profiling}).
We conduct an experimental study using the Attention--AllReduce partition of Llama 3.2 3B on 8 NVIDIA A100 GPUs\footnote{Batch size is 4 and sequence length is 4K, with tensor parallelism degree 8, running at 1,410 MHz.} to analyze the impact of measurement window and cooldown duration on the energy measurement accuracy.

\begin{figure}[t]
    
    \subfloat[Measurement duration\label{fig:profiling-duration}]{
        \includegraphics[width=0.48\columnwidth]{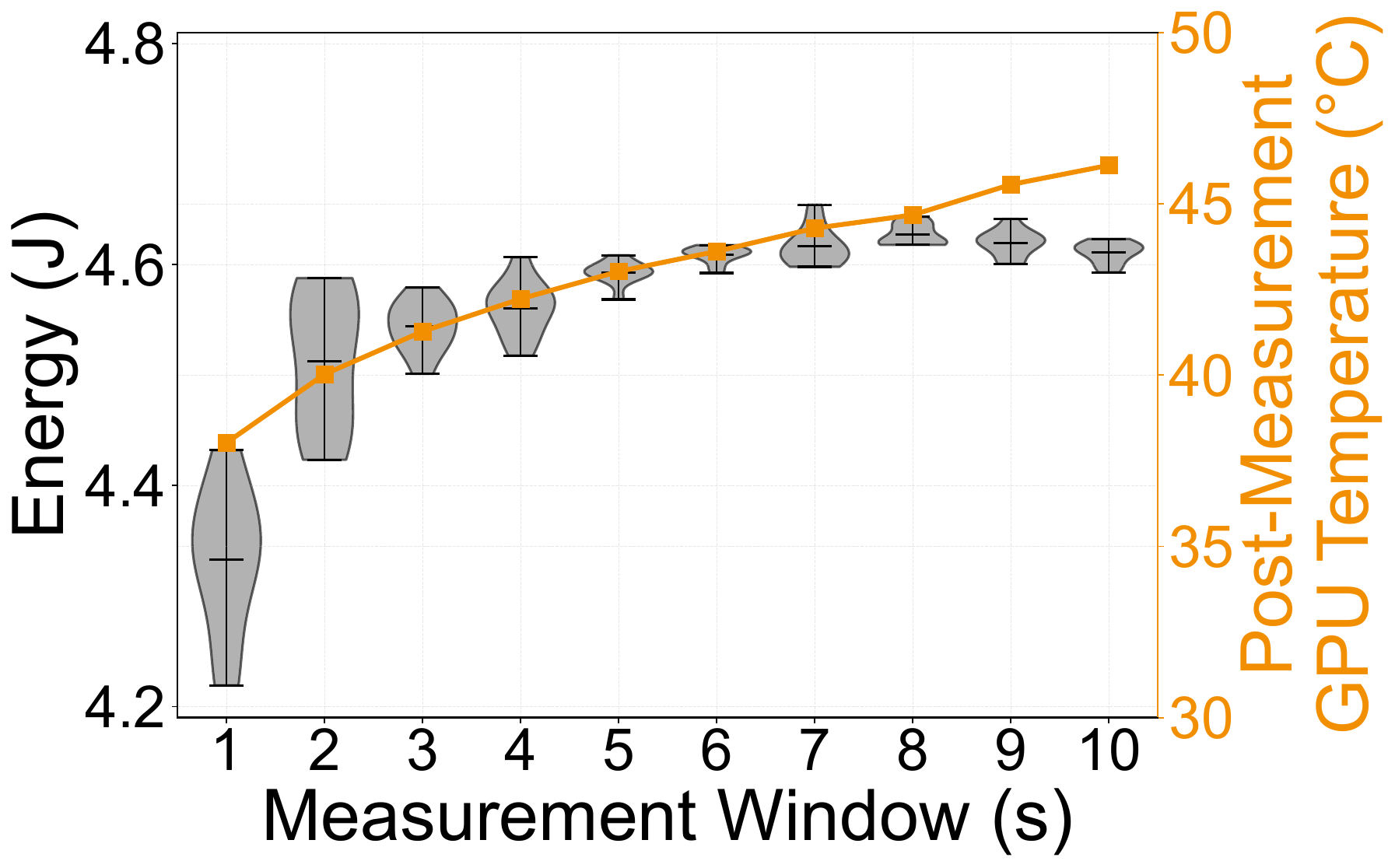}
    }
    \subfloat[Cooldown duration\label{fig:profiling-cooldown}]{
        \includegraphics[width=0.48\columnwidth]{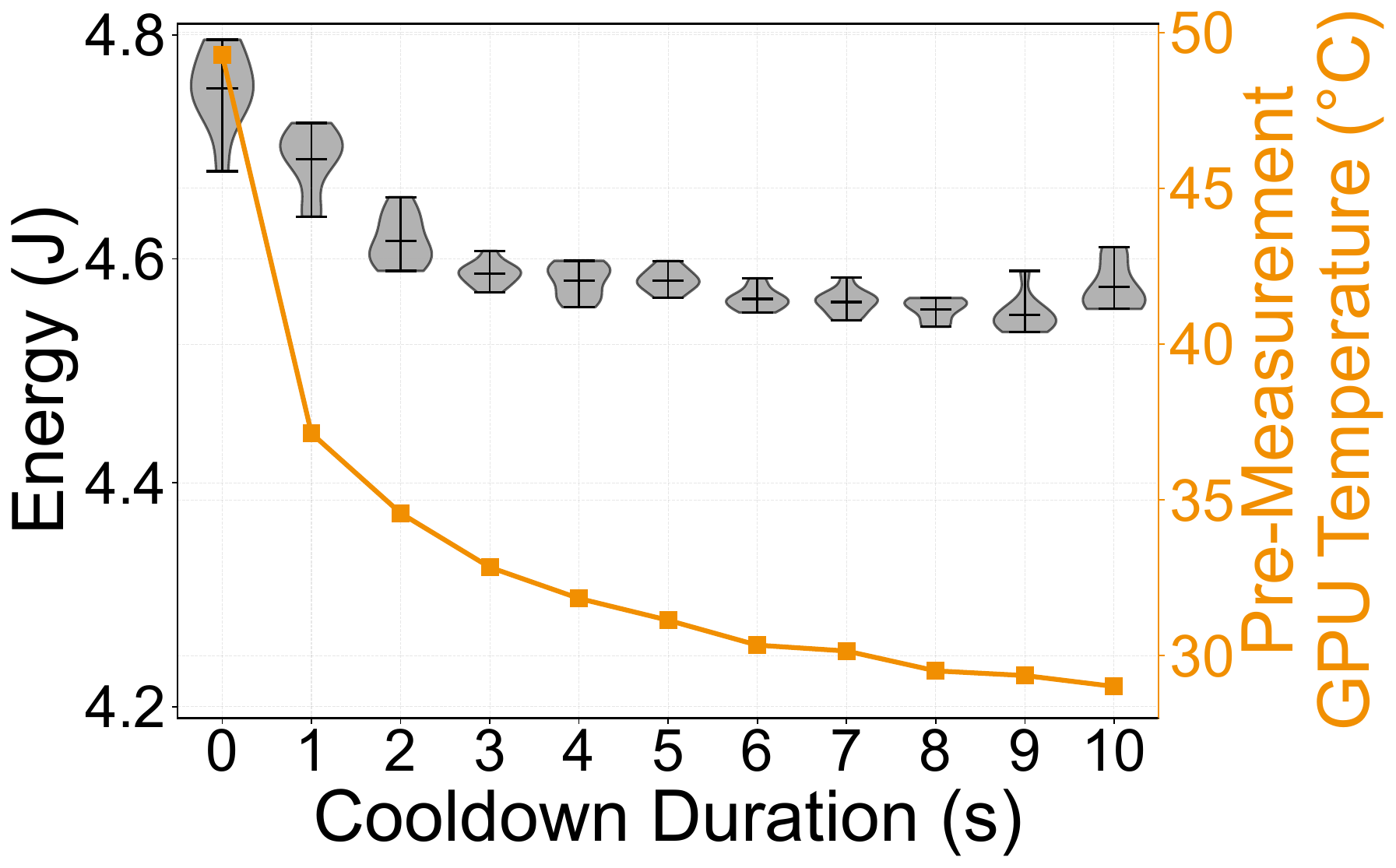}
    }
    \caption{\textbf{Impact of changing (a) measurement duration and (b) cooldown duration for the Thermally Stable Profiler. We report the distribution of energy across 10 repeated trials, along with the average GPU temperature before and after measurement.}}
    \label{fig:profiling}
\end{figure}

\paragraph{Measurement duration.}
We fix the cooldown duration to 5 seconds and vary the measurement window from 1 second to 10 seconds, each with 10 repeated profiling trials.
Figure~\ref{fig:profiling-duration} shows the distribution of measured energy consumption under different measurement windows, together with the average GPU temperature \emph{after} each measurement.
With very short measurement windows (e.g., below 2 seconds), energy measurements exhibit large variability due to the 100 ms NVML counter update interval, and have a lower mean because the GPU has not fully warmed up.
Energy measurement stabilizes from 5 seconds onward, which is why we chose 5 seconds as our measurement duration.

\paragraph{Cooldown duration.}
Cooldown duration is the waiting time between profiling consecutive partitions.
We fix the measurement duration to 5 seconds and vary the cooldown duration from 0 second (no cooldown) to 10 seconds, with 10 repeated profiling trials.
Figure~\ref{fig:profiling-cooldown} shows the distribution of measured energy consumption under different cooldown durations, along with the average GPU temperature \emph{before} each measurement.
We observe that mean energy consumption strongly correlates with the GPU temperature at the start of measurement, making sufficient cooldown essential for accurate measurement.
Temperature and measurements stabilize from 5 seconds onward, so we cool down for 5 seconds between partitions.
We note that the optimal cooldown duration would require tuning for different computing environments.

  \section{Related Work}\label{sec:related}

\subsection{ML Energy Optimization}

The energy consumption of ML workloads has attracted extensive attention, prompting a growing body of work on energy measurement and optimization~\cite{dynamollm-hpca25,mlenergy-benchmark-neurips25,mlenergy-v3-arxiv26,muserve-atc24,multirequest-energy-pomacs26,zeus-nsdi23,perseus-sosp24,envpipe-atc23,llm-training-energy-euromlsys26}.
In particular, the time--energy tradeoff frontier has been a key tool for reasoning about performance and energy efficiency.
For LLM serving, DynamoLLM~\cite{dynamollm-hpca25} explores time--energy tradeoffs via model parallelism and GPU frequency scaling under latency constraints. 
The ML.ENERGY benchmark~\cite{mlenergy-benchmark-neurips25} characterizes the time--energy frontier across deployment configurations to guide latency-aware energy optimization. 
On the training side, Zeus~\cite{zeus-nsdi23} identifies job-level time--energy tradeoffs on recurring training workloads.
Perseus~\cite{perseus-sosp24} dynamically scales GPU frequency per microbatch to optimize energy under pipeline iteration time deadlines.
However, these approaches often yield suboptimal frontiers because they overlook the impact of low-level execution scheduling, such as communication and computation overlap.
Jayaweera et al.~\cite{tile-selection-cgo24} propose energy-aware tile size selection for GPU kernels, but do not consider interactions with system-level latency targets, limiting the potential for energy optimization.

\subsection{Communication Scheduling}

As model scale grows, communication incurs increasingly significant overhead, motivating extensive work on overlapping communication with computation~\cite{nanoflow-osdi25,domino-arxiv24,concerto-asplos25,centauri-asplos24,coconet-asplos22,overlap-decomposition-asplos23,tilelink-mlsys26,tritondistributed-arxiv25,syncopate-osdi26,flux-arxiv24,tetriserve-asplos26,deepseek-v3-arxiv24,tokenweave-mlsys26,vllm-dbo-docs,megascale-nsdi24,megatron-tp-overlap}. 
Nanobatching is a key technique that removes data dependencies and enables overlap opportunities. 
In LLM serving, NanoFlow~\cite{nanoflow-osdi25}, TokenWeave~\cite{tokenweave-mlsys26}, and vLLM Dual Batch Overlap~\cite{vllm-dbo-docs} improve throughput via nanobatching with fine-grained kernel overlap. 
For training, DeepSpeed Domino~\cite{domino-arxiv24} applies nanobatching to overlap tensor-parallel communication, while DeepSeek DualPipe~\cite{deepseek-v3-arxiv24} pipelines forward and backward nanobatches for communication overlap. 
Concerto~\cite{concerto-asplos25} optimizes communication scheduling from a compiler perspective and adopts partial nanobatching when data dependencies are present.
Other hand-crafted kernels fuse communication into computation~\cite{flux-arxiv24,tilelink-mlsys26,tritondistributed-arxiv25,megatron-tp-overlap}, but are tightly tied to specific parallelism strategies and operators.
For example, Megatron-LM supports tensor-parallel communication overlap~\cite{megatron-tp-overlap}, which relies on its native sequence parallelism~\cite{activation-recomputation-mlsys23}.
From the energy perspective, these approaches reduce static energy by reducing time, but do not explicitly optimize energy in their design.
Kareus reveals the energy impact of communication scheduling and generalizes the nanobatching strategy to enable fine-grained SM allocation and launch-timing control, further minimizing static energy while jointly reducing dynamic energy by optimizing GPU frequency.

  \section{Conclusion}

We present Kareus, an execution-aware energy optimizer for large model training.
Kareus demonstrates that SM allocation, kernel launch timing, and GPU frequency jointly influence static and dynamic energy, and that the optimal time--energy frontier cannot be attained by optimizing them in isolation.
Building on this insight, Kareus introduces a partitioned overlap execution model and a multi-objective Bayesian optimization framework that jointly searches over execution schedules, exposing a strictly better time--energy tradeoff frontier than prior systems.

Looking forward, we believe energy-aware execution scheduling should be a first-class concern in ML systems, not an afterthought.
As AI scaling continues and energy constraints tighten for gigawatt-scale AI datacenters, systems that co-optimize computation, communication, and power will be essential, not just for cost savings but for enabling training runs that would otherwise be infeasible.
Kareus takes a step in this direction, and we hope it spurs further work on energy as a core system design metric.

  \section*{Acknowledgements}
We would like to thank the OSDI reviewers, our shepherd, and SymbioticLab members for their insightful feedback.
This work was supported in part by NSF grants CCF-2450085 and CNS-2106184, DARPA ML2P Award HR0011-26-9-E190, and by grants from Cisco, Ford, Mozilla Foundation, and Laude Institute.
Jae-Won Chung is additionally supported by the Kwanjeong Educational Foundation.

  \label{EndOfPaper}

	\balance
	{
		\bibliographystyle{plain}
		\bibliography{ref}
	}
	\clearpage
	\nobalance

	\appendix

\section{Energy Efficiency of Constant Frequency}\label{apdx:frequency-theorem}

Below, we formalize why steady GPU frequency minimizes energy at a fixed average frequency.

\begin{theorem}[Energy Efficiency of Constant Frequency]
Let $f(t)$ denote the GPU frequency over a time interval $[0, T]$, and let $\bar{f} = \frac{1}{T} \int_0^T f(t) \, dt$ be its time-average. Under the following assumptions:
\begin{enumerate}
    \item Dynamic power scales cubically with frequency: $P_{\text{dyn}}(t) = k \cdot f(t)^3$ for some constant $k > 0$.
    \item Static power is constant: $P_{\text{static}}(t) = P_s$ for some constant $P_s \geq 0$.
    \item Execution time depends only on average frequency: workloads with the same average frequency $\bar{f}$ complete in the same time $T$. This is supported by empirical observations from Kareus optimization results.
\end{enumerate}
Then the total energy consumption is minimized when frequency is held constant at $\bar{f}$.
\end{theorem}

\begin{proof}
Total energy is the sum of dynamic and static energy:
\[
E_{\text{total}} = E_{\text{dyn}} + E_{\text{static}} = \int_0^T k \cdot f(t)^3 \, dt + P_s \cdot T.
\]

By Assumption~3, $T$ is identical for both the fluctuating and constant frequency cases. Thus, static energy $E_{\text{static}} = P_s \cdot T$ is equal in both cases.

For dynamic energy, consider the function $g(x) = x^3$, which is convex for $x > 0$ since $g''(x) = 6x > 0$. By Jensen's Inequality:
\[
g(\bar{f}) = \bar{f}^3 \leq \frac{1}{T} \int_0^T f(t)^3 \, dt.
\]

Multiplying both sides by $kT$:
\[
E_{\text{dyn}}^{\text{constant}} = kT \cdot \bar{f}^3 \leq k \int_0^T f(t)^3 \, dt = E_{\text{dyn}}^{\text{fluctuating}}.
\]

Equality holds if and only if $f(t) = \bar{f}$ almost everywhere, i.e., when frequency is constant. Therefore:
\[
E_{\text{total}}^{\text{constant}} \leq E_{\text{total}}^{\text{fluctuating}},
\]
with strict inequality when $f(t)$ varies over time.
\end{proof}

\paragraph{Implications for GPU frequency scaling.}
This theorem explains why frequency fluctuation due to power limits can be energy-inefficient.
When a GPU operates at a high frequency and triggers power throttling, the instantaneous frequency fluctuates, while the time-averaged frequency remains close to what could be achieved by operating steadily at a lower frequency.
Consequently, the execution time---and thus static energy consumption---are the same as those under steady lower-frequency operation.
However, this fluctuation incurs higher dynamic energy consumption because dynamic power is a strictly convex function of frequency. 
By Jensen's Inequality, any variance in frequency increases the expected dynamic energy consumption.

\section{Global Solution Space}\label{apdx:launch-timing-dp}

The global solution space of combining GPU frequency, SM allocation, and kernel launch timing is extremely large. 
In this section, we illustrate the space using the NVIDIA A100 GPU as an example.

\paragraph{GPU frequencies.}
Supported GPU frequencies by A100 are from 210 MHz to 1,410 MHz at a stride of 15 MHz.
We restrict the search space to 900 MHz--1,410 MHz (35 choices) because lowering the frequency below 900 MHz no longer reduces energy.\footnote{Power reduction slows down but time increases, leading to higher energy.}

\paragraph{SM allocations.}
A100 has 108 SMs, and we restrict the search space to up to 30 SMs because allocating additional SMs beyond this empirically no longer improved communication latency.
This leaves 30 choices for SM allocation.

\paragraph{Kernel launch timing.}
Finally, given a sequence of computation and communication kernels, possible execution orders can be expressed as a recurrence relation, which enumerates the number of subproblems.
We elaborate on this formulation in the following paragraphs.
In summary, for a typical LLM composed of Transformer blocks, this can result in 81 possible groupings, leading to a \emph{global solution space} with in total 85,050 candidates.

\paragraph{Formulation of launch timing.}
Kareus generalizes the nanobatching scheme to eliminate dependencies between computation and communication, which enables overlap and results in two operation sequences:
\[
\mathcal{S}_1 = \{ O_0, O_1, \dots, O_i \}, 
\qquad
\mathcal{S}_2 = \{ O_0', O_1', \dots, O_j' \}.
\]
where operations within the same sequence must be executed in order, while operations across sequences are independent and may overlap.
We regard the consecutive communication operations as a single operation.
The time--energy frontier of the launch timing schedule can be formulated as a dynamic programming (DP) recurrence:
\[
\begin{aligned}
\mathcal{P}(i,j)
&= \min_{\text{Pareto}} \\
&\left\{
\begin{aligned}
& (E(O_i) + E(\mathcal{P}(i+1,j)),\; T(O_i) + T(\mathcal{P}(i+1,j))) \\[6pt]
& (E(O_j') + E(\mathcal{P}(i,j+1)),\; T(O_j') + T(\mathcal{P}(i,j+1))) \\[6pt]
& (E(O_i \parallel O'_{j..j+k}),\; T(O_i \parallel O'_{j..j+k})) + \mathcal{P}(i+1, j+k) \\[6pt]
& (E(O_j' \parallel O_{i..i+k}),\; T(O_j' \parallel O_{i..i+k})) + \mathcal{P}(i+k, j+1)
\end{aligned}
\right.
\end{aligned}
\]

Here, $O_i \parallel O'_{j..j+k}$ means overlapping operation $O_i$ with a subsequence $\{O'_j,...,O'_{j+k}\}$; the notation is symmetric for the opposite case, $O'_j \parallel O_{i..i+k}$.
The cost functions $E(\cdot)$ and $T(\cdot)$ account for the interference when overlapping.

This recurrence enumerates the time and energy of all single operations and all feasible overlap patterns between the two sequences. For a typical Transformer block with 9 computation operations and 1 AllReduce, if we restrict overlap to occur only between communication and computation and cap the maximum overlap length at 9 (assuming the communication is no longer than a full Transformer block), there are 81 possible overlap patterns. Including the non-overlapped cases, this yields a total of 91 subproblems.

\section{Adaptation of MBO Hyperparameters}\label{apdx:mbo-hyperparameters}
In this section, we describe how Kareus configures MBO hyperparameters. Some are adapted across different partitions based on partition complexity, while others remain fixed.

\paragraph{Search space.}
For GPU frequency, we restrict the search space to 900–1,410 MHz with a stride of 30 MHz, since the time and energy differences between adjacent 15 MHz settings are marginal.
For SM allocation, the search space is determined by the communication group size. If the number of GPUs in a communication group is fewer than 4, we search SM allocations from 1 to 20 with a stride of 1. If the group size is 4 or larger (e.g., the common cases of 4 and 8), we search SM allocations from 3 to 30 with a stride of 3.
For launch timing, we enumerate all computation operators within a partition and exclude options that always lead to exposed communication, such as launching AllReduce from Linear2 in Figure~\ref{fig:analysis-vary-sm-2}.

\paragraph{Sample size.}
Our goal is to control the total profiling time while preserving the quality of the time--energy frontier. We adopt a relatively large initialization set for informative exploration and scale the batch search budget based on partition complexity. 
Specifically, we classify partitions into three categories: \emph{small} partitions containing only one computation, \emph{medium} partitions containing 2–3 computations, and \emph{large} partitions containing more than three computations.
We then configure the sample sizes as follows: initial sample size $N_{\text{init}} = 36$, maximum number of batches $B_{\max} = 3$, batch size $k = 16$ for small partition; $N_{\text{init}} = 48$, $B_{\max} = 4$, $k = 16$ for medium partition; $N_{\text{init}} = 96$, $B_{\max} = 4$, $k = 32$ for large partition.
We set the proportions of the total energy pass, dynamic energy pass, static energy pass, and uncertainty pass to 0.4, 0.2, 0.2, and 0.2, respectively.
As shown in Section~\ref{sec:evaluation-bo}, the MBO overhead of Kareus is controlled within two hours on average. 

\paragraph{XGBoost hyperparameters.}
Since the configuration space is three-dimensional, we adopt XGBoost hyperparameter settings commonly used for low-dimensional regression. To ensure fast and stable convergence while avoiding overfitting, we use relatively shallow trees with $max\_depth = 6$, a high learning rate $\eta = 0.3$, and a modest model capacity with $num\_boost\_round = 100$.
For the bootstrap ensemble, we set the ensemble size to 5, use a bootstrap sampling fraction of 0.8, and vary the random seed across bootstrap resamples.

\paragraph{HV reference point.}
At each batch iteration, we compute the HV reference point using values slightly worse than the worst observed measurements to ensure boundedness. Specifically, we set $r = \bigl(1.1 \times \max T(x),\; 1.1 \times \max E(x)\bigr)$.

\paragraph{Stopping conditions.}
The MBO algorithm is terminated when either a sufficient number of samples has been collected or the objective has converged, according to the following criteria:
(1) Stop after a fixed number of batches \(B_{\max}\).
(2) After each batch, compute the dominated hypervolume of the measured frontier. Stop if the moving average of the relative HV improvement over the last \(R\) batches falls below \(\varepsilon\).

We set the window size in the stopping criterion to $R = 2$ and the convergence threshold to $\varepsilon = 10^{-3}$. Energy and time are normalized throughout the MBO process, ensuring stable stopping behavior across partitions.

\section{End-to-End Results}\label{apdx:frontier-e2e}

Figure~\ref{fig:apdx-frontiers-testbed} shows the iteration time--energy frontiers measured on the testbed GPUs for all model configurations, complementing the evaluation in Section~\ref{sec:eval-e2e-frontier}.

\begin{figure}[t!]
  \centering
  \includegraphics[width=0.85\columnwidth]{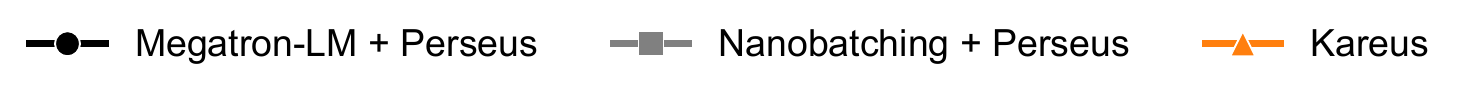}\\[-1.25em]
  \subfloat[Llama 3.2 3B, CP=1, TP=8,\\$\mu$BS=8, Seq=4096]{
    \includegraphics[width=0.49\columnwidth]{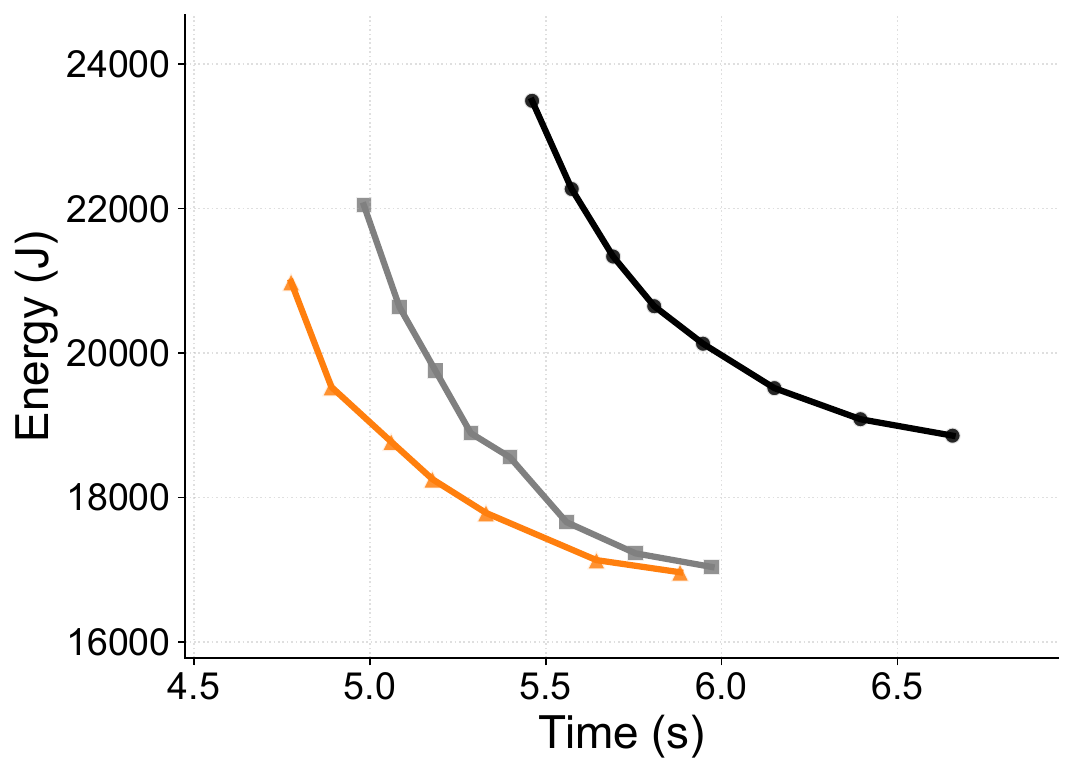}
  }%
  \subfloat[Llama 3.2 3B, CP=2, TP=4,\\$\mu$BS=8, Seq=4096]{
    \includegraphics[width=0.49\columnwidth]{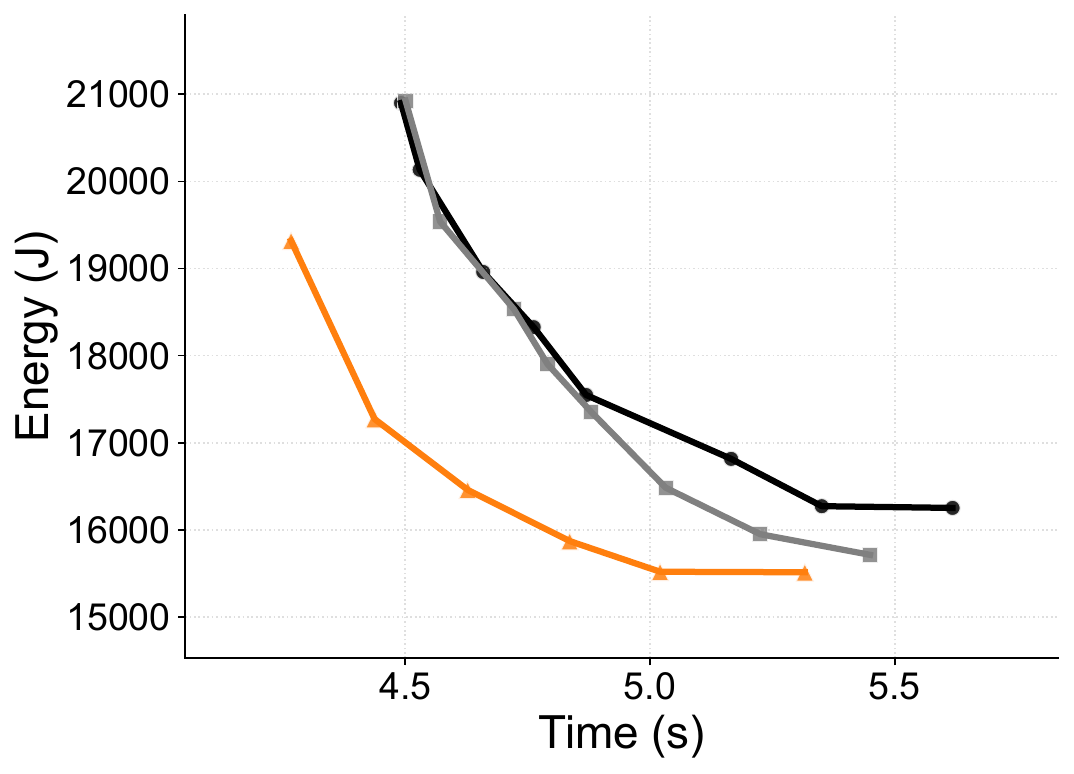}
  }\\
  \subfloat[Llama 3.2 3B, CP=2, TP=4,\\$\mu$BS=8, Seq=8192]{
    \includegraphics[width=0.49\columnwidth]{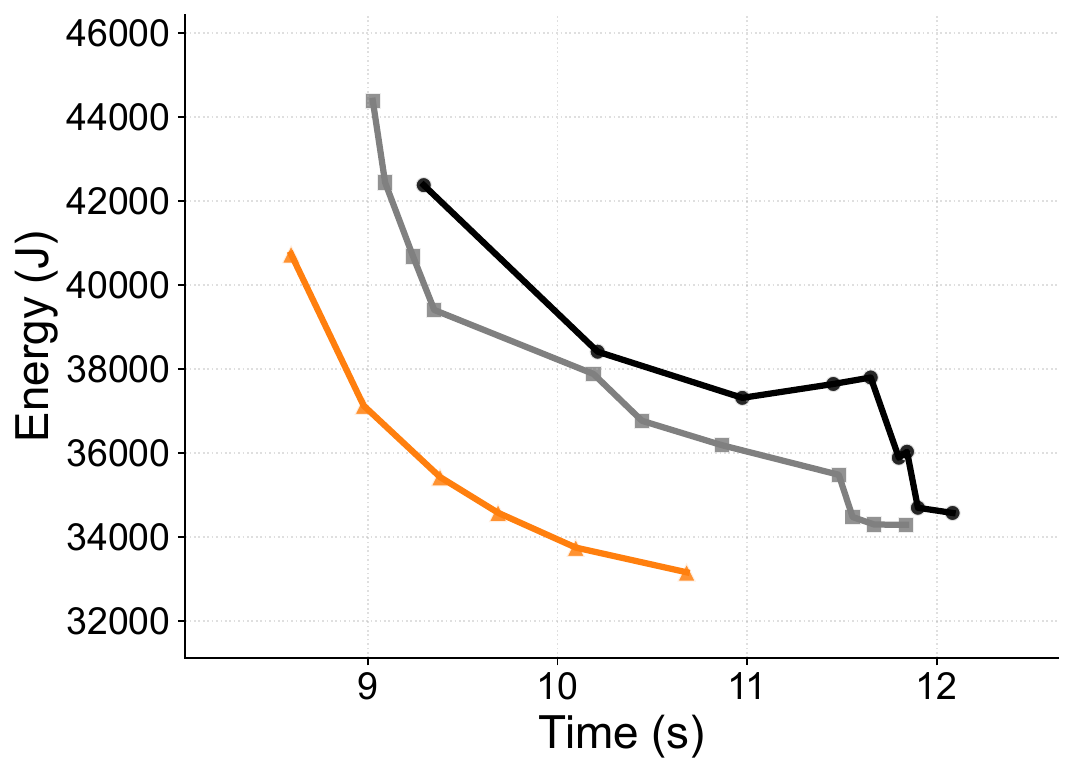}
  }%
  \subfloat[Llama 3.2 3B, CP=2, TP=4,\\$\mu$BS=16, Seq=4096]{
    \includegraphics[width=0.49\columnwidth]{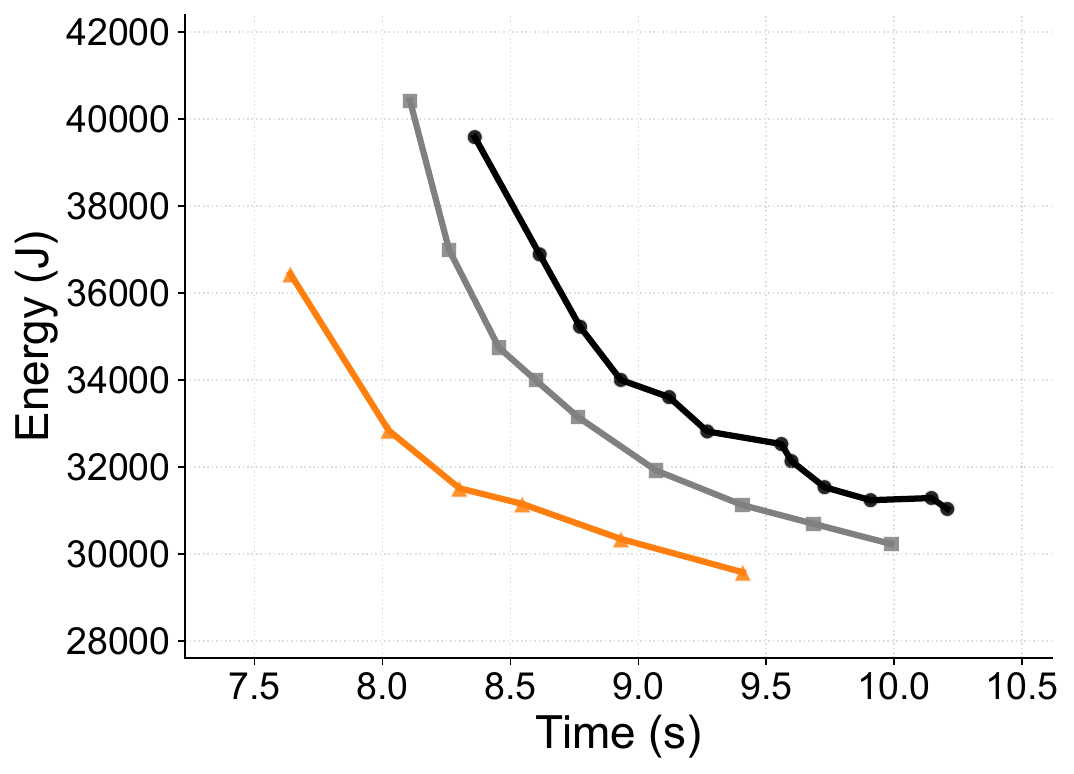}
  }\\
  \subfloat[Microbatch size = 8]{
    \includegraphics[width=0.49\columnwidth]{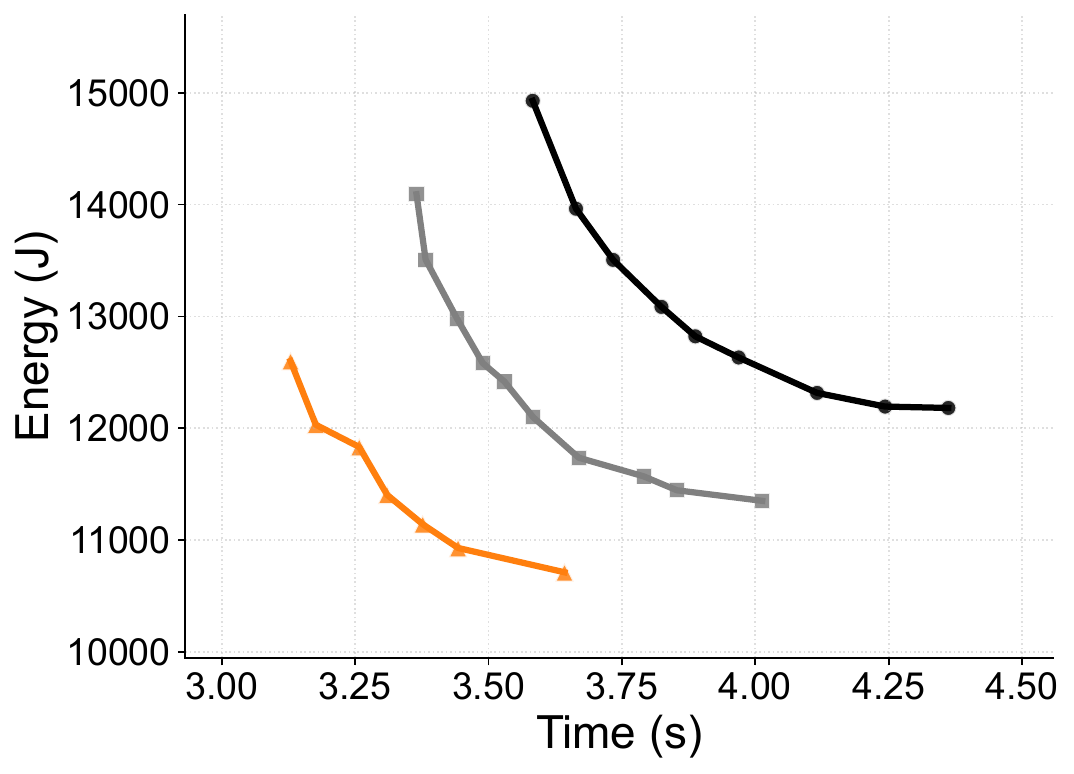}
  }%
  \subfloat[Qwen 3 1.7B, CP=1, TP=8,\\$\mu$BS=8, Seq=8192]{
    \includegraphics[width=0.49\columnwidth]{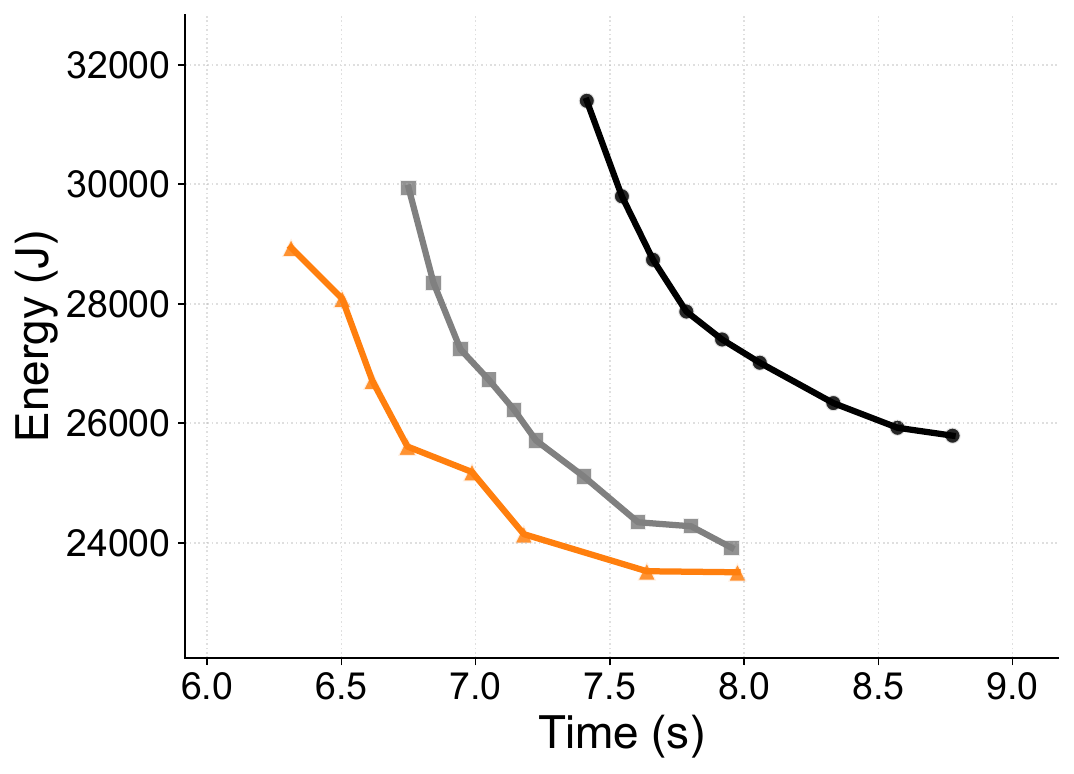}
  }\\
  \subfloat[Qwen 3 1.7B, CP=1, TP=8,\\$\mu$BS=16, Seq=4096]{
    \includegraphics[width=0.49\columnwidth]{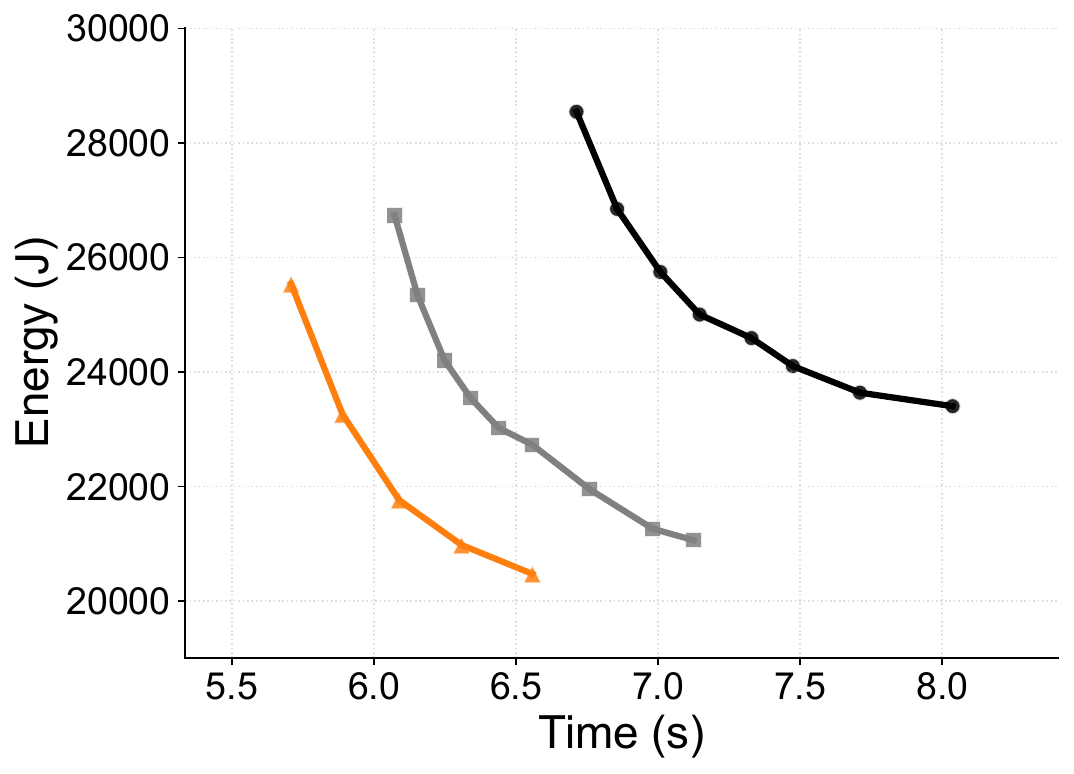}
  }%
  \subfloat[Qwen 3 1.7B, CP=2, TP=4,\\$\mu$BS=8, Seq=4096]{
    \includegraphics[width=0.49\columnwidth]{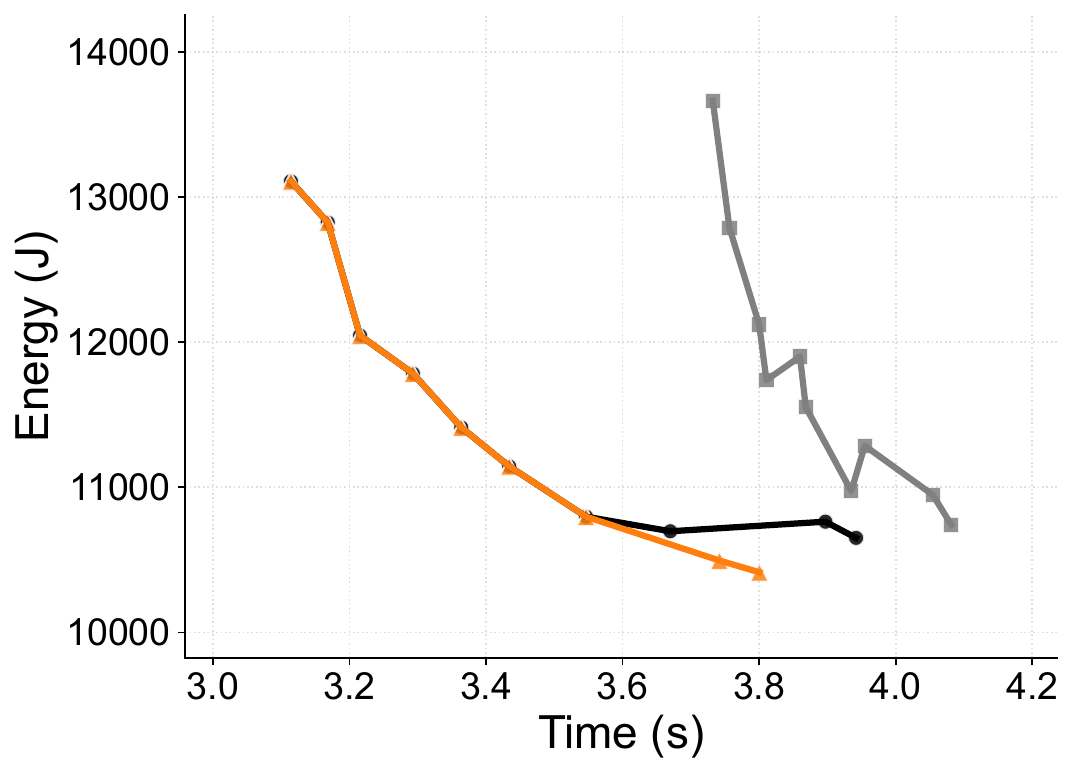}
  }\\
  \subfloat[Qwen 3 1.7B, CP=2, TP=4,\\$\mu$BS=8, Seq=8192]{
    \includegraphics[width=0.49\columnwidth]{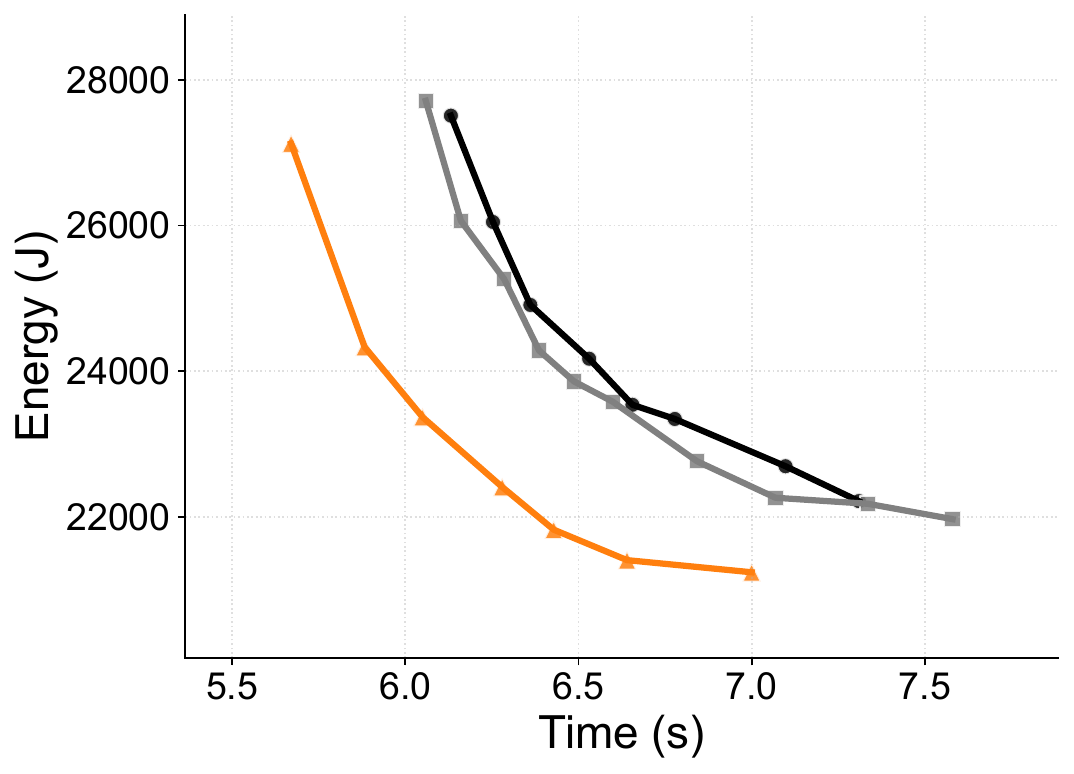}
  }%
  \subfloat[Qwen 3 1.7B, CP=2, TP=4,\\$\mu$BS=16, Seq=4096]{
    \includegraphics[width=0.49\columnwidth]{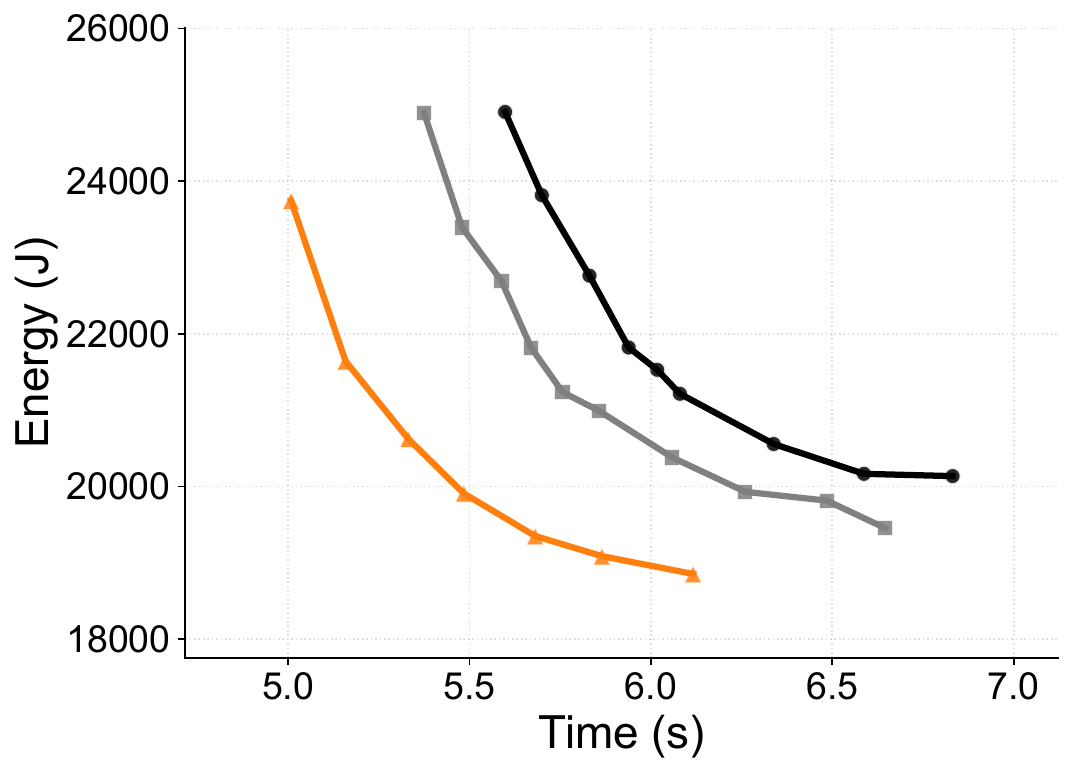}
  }
  \caption{
    [Experiment] Itertaion time--energy frontiers of Megatron-LM + Perseus, Nanobatching + Perseus, and Kareus for all model configurations. CP: Context Parallelism, TP: Tensor Parallelism, $\mu$BS: Microbatch Size, Seq: Sequence Length.
  }\label{fig:apdx-frontiers-testbed}
\end{figure}

\section{Large-Scale Emulation}\label{apdx:frontier-emulation}

Figure~\ref{fig:apdx-frontiers-emulation} shows the iteration time--energy frontiers in large-scale emulation for Llama 3.3 70B across different numbers of microbatches, complementing the evaluation in Section~\ref{sec:evaluation-emulation}.

\begin{figure}[t!]
  \centering
  \includegraphics[width=0.55\columnwidth]{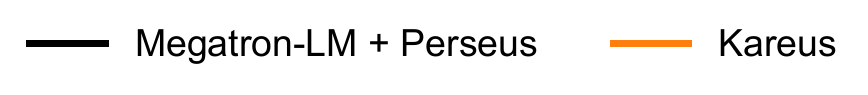}\\[-1em]
  \subfloat[\# Microbatches = 16]{
    \includegraphics[width=0.49\columnwidth]{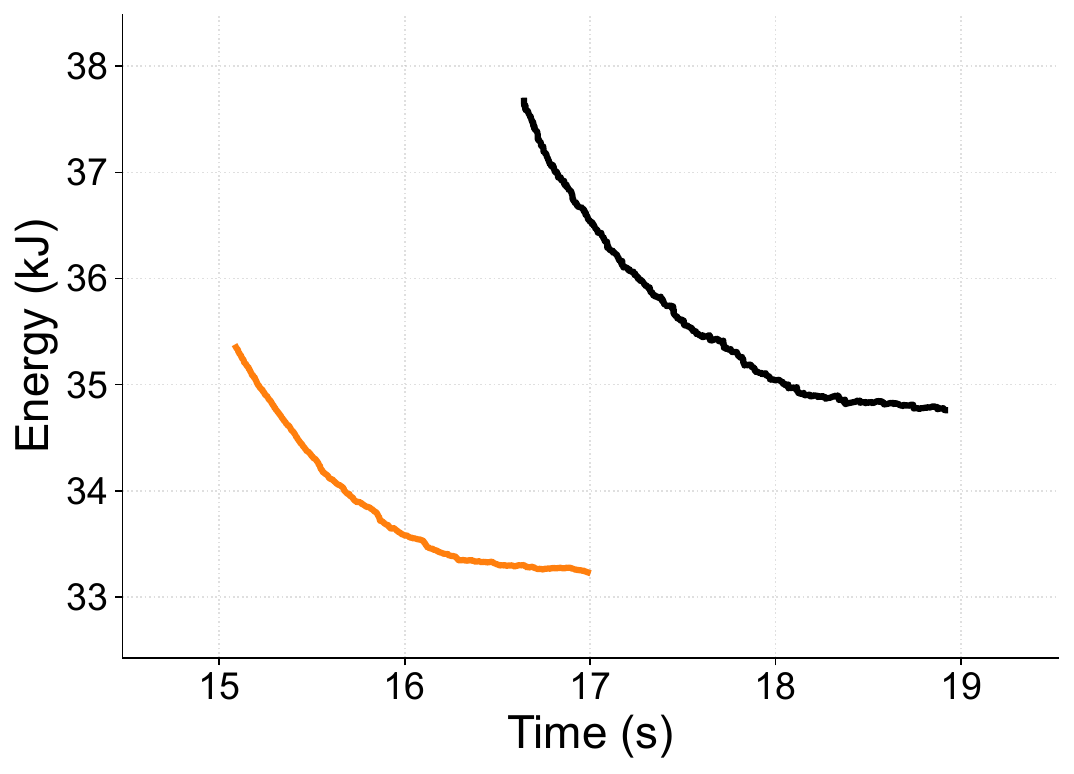}
  }\label{fig:emulation-mb16}%
  \subfloat[\# Microbatches = 32]{
    \includegraphics[width=0.49\columnwidth]{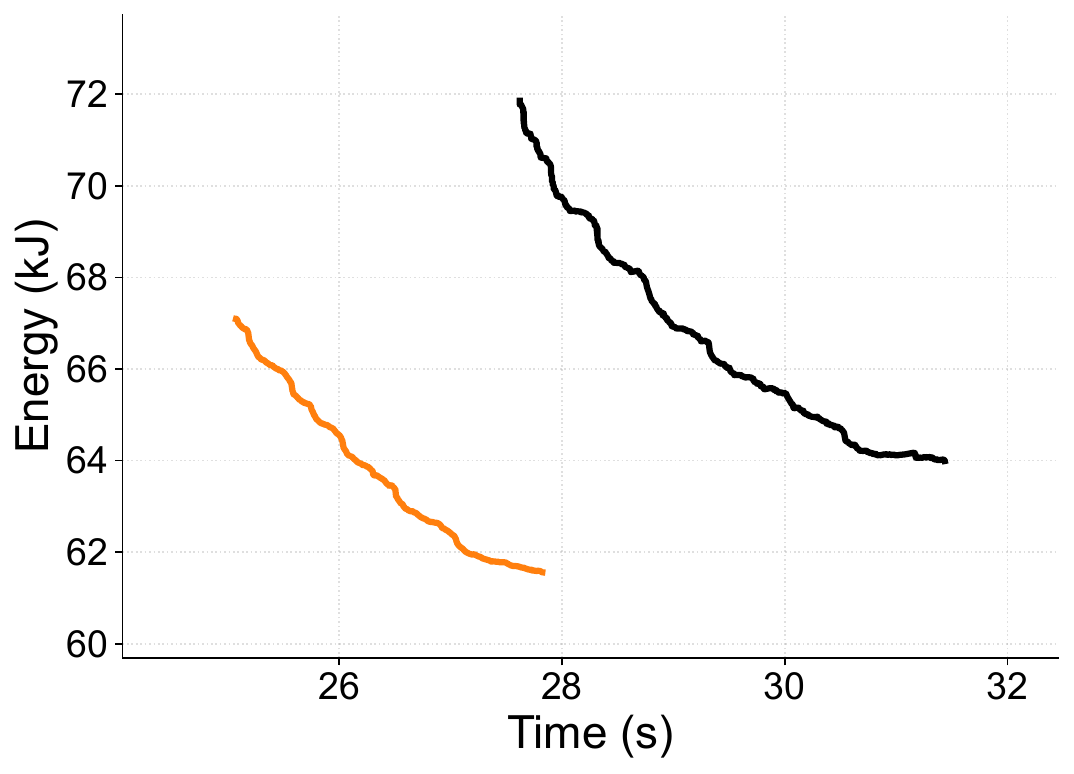}
  }\label{fig:emulation-mb32}\\
  \subfloat[\# Microbatches = 64]{
    \includegraphics[width=0.49\columnwidth]{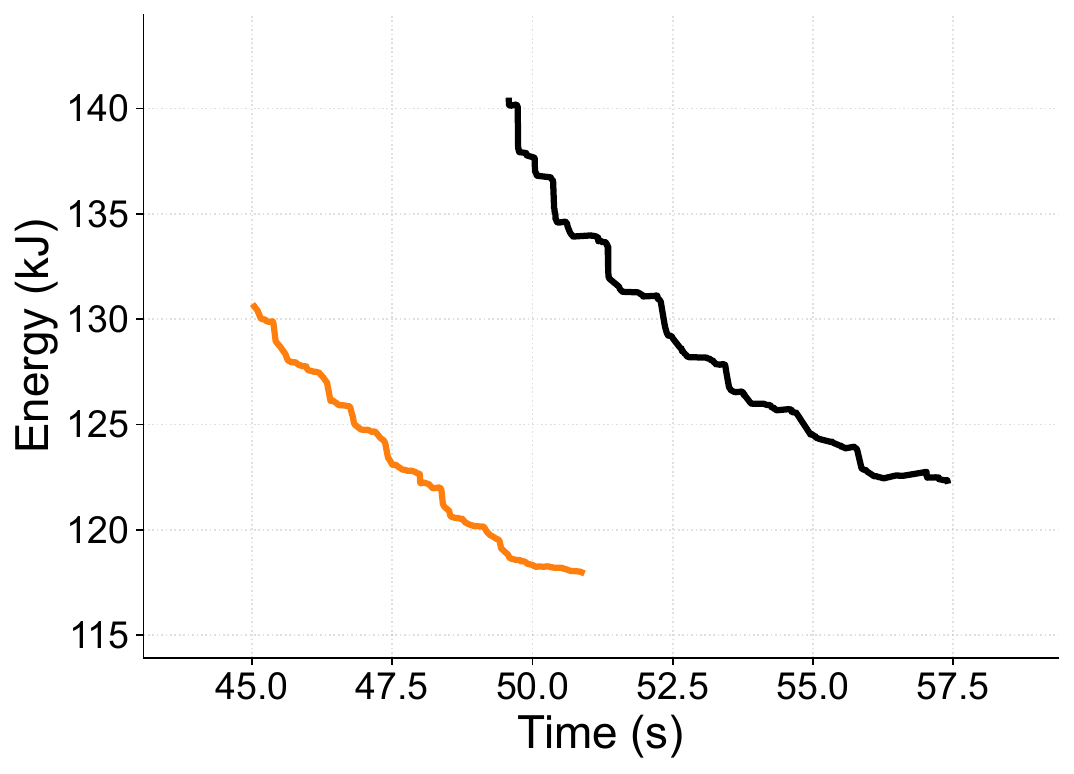}
  }\label{fig:emulation-mb64}%
  \subfloat[\# Microbatches = 128]{
    \includegraphics[width=0.49\columnwidth]{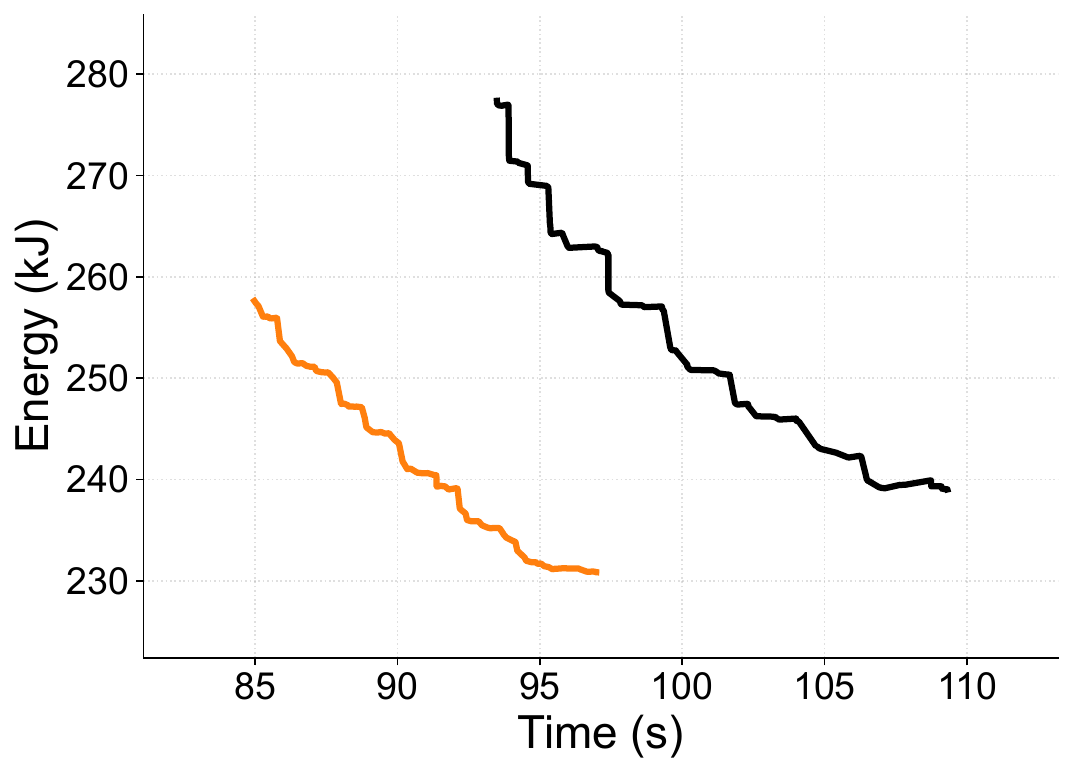}
  }\label{fig:emulation-mb128}
  \caption{
    [Emulation] Iteration time--energy frontiers of Megatron-LM + Perseus and Kareus for Llama 3.3 70B across different numbers of microbatches.
  }\label{fig:apdx-frontiers-emulation}
\end{figure}

\section{Sensitivity to Microbatch Size}\label{apdx:frontier-sensitivity}

Figure~\ref{fig:apdx-frontiers-sensitivity} shows the iteration time--energy frontiers measured on the testbed GPUs for Qwen 3 1.7B with TP=8 and sequence length 4K across different microbatch sizes, complementing the sensitivity study in Section~\ref{sec:evaluation-sensitivity}.

\begin{figure}[t!]
  \centering
  \includegraphics[width=0.55\columnwidth]{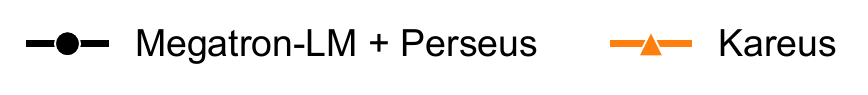}\\[-1em]
  \subfloat[Microbatch size = 8]{
    \includegraphics[width=0.49\columnwidth]{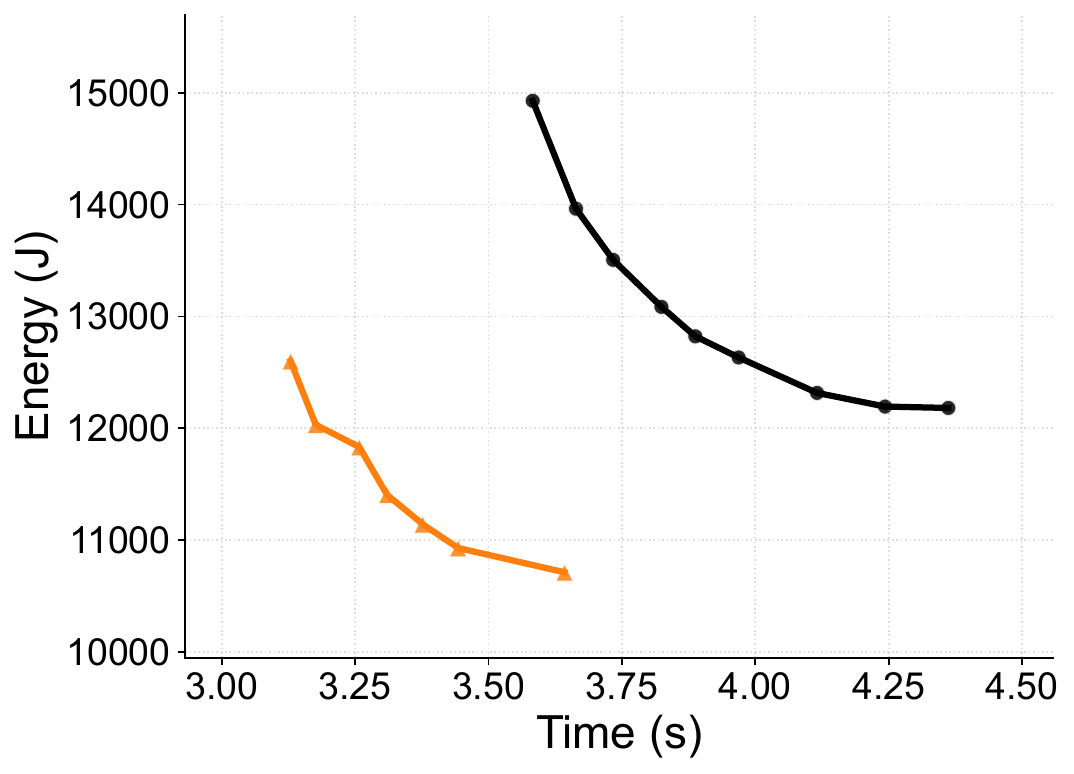}
  }\label{fig:sensitivity-bs8}%
  \subfloat[Microbatch size = 12]{
    \includegraphics[width=0.49\columnwidth]{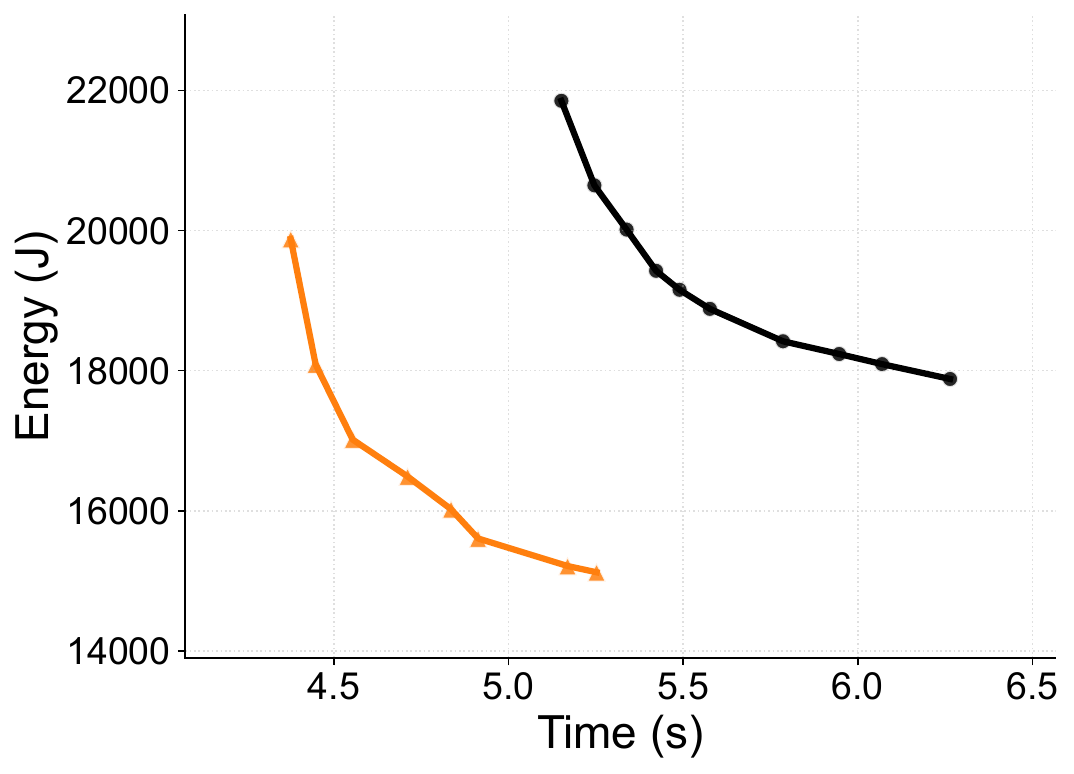}
  }\label{fig:sensitivity-bs12}\\
  \subfloat[Microbatch size = 16]{
    \includegraphics[width=0.49\columnwidth]{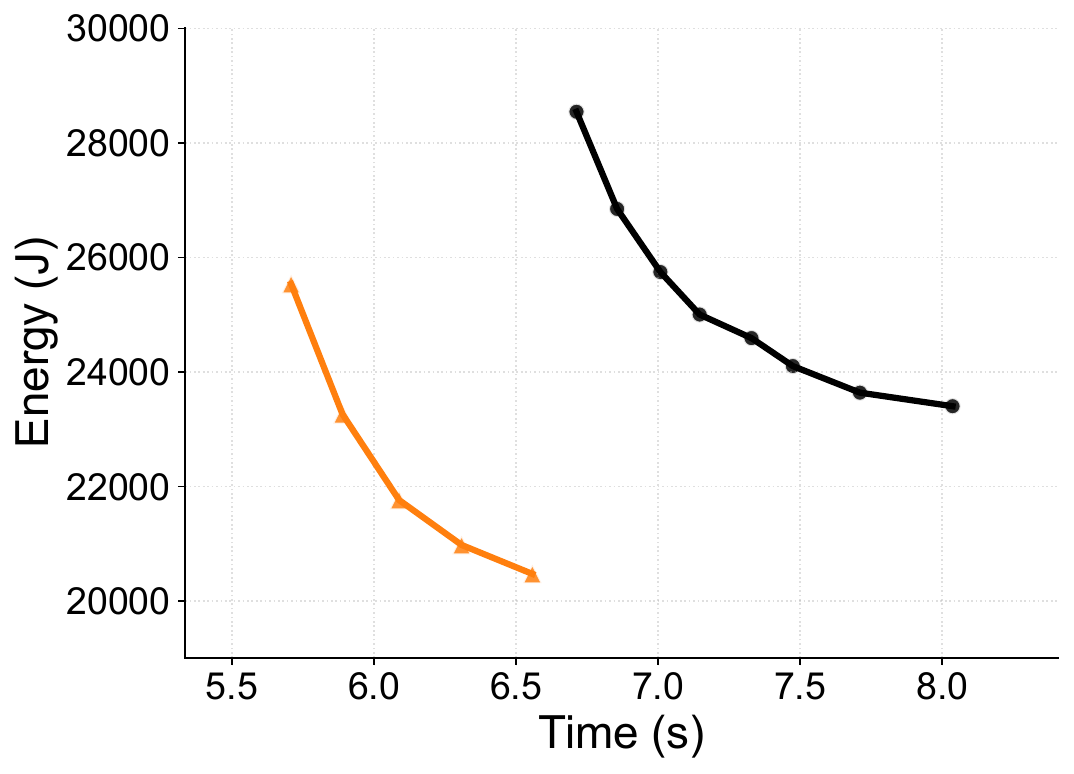}
  }\label{fig:sensitivity-bs16}%
  \subfloat[Microbatch size = 20]{
    \includegraphics[width=0.49\columnwidth]{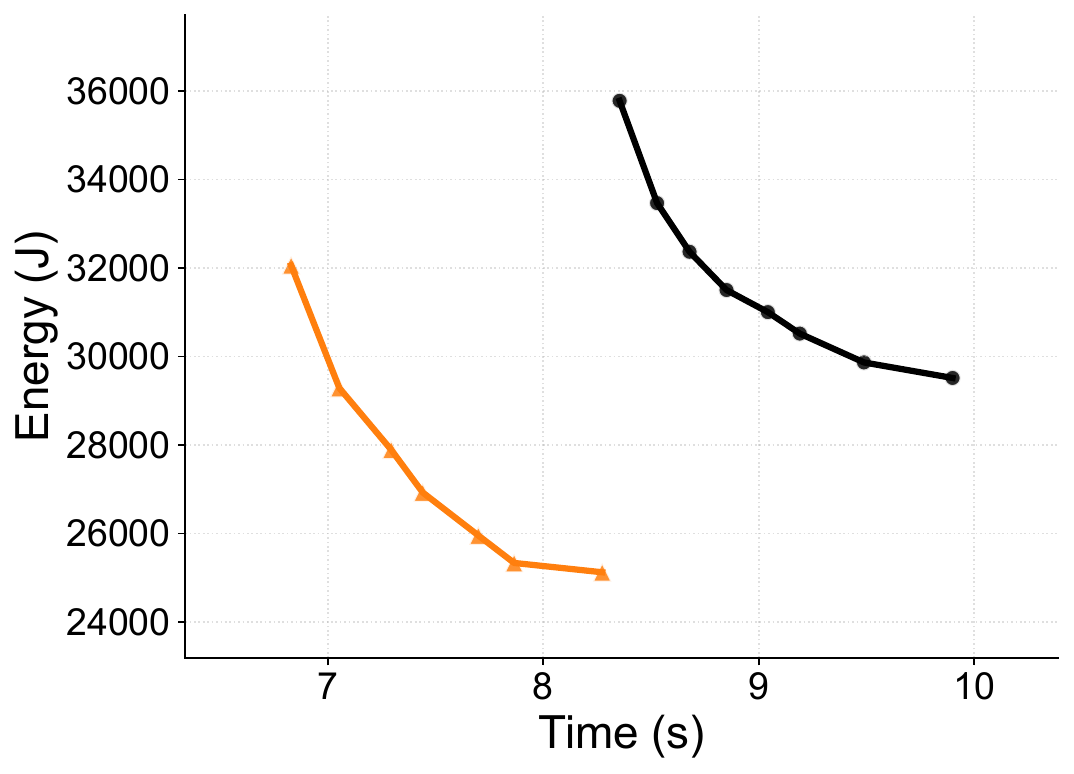}
  }\label{fig:sensitivity-bs20}
  \caption{
    [Experiment] Iteration time--energy frontiers of Megatron-LM + Perseus and Kareus for Qwen 3 1.7B across different microbatch sizes.
  }\label{fig:apdx-frontiers-sensitivity}
\end{figure}

\end{document}